\documentclass{article}
\usepackage{arxiv}
\usepackage[toc,page]{appendix}

\usepackage{inputenc} %
\usepackage[T1]{fontenc}    %
\usepackage{hyperref}       %
\usepackage{url}            %
\usepackage{booktabs}       %
\usepackage{amsfonts}       %
\usepackage{nicefrac}       %
\usepackage{microtype}      %
\usepackage{lipsum}
\usepackage{mathtools}
\usepackage{cite}
\usepackage{amsmath}
\usepackage{amssymb}
\usepackage{amsthm}
\usepackage{algorithm,algcompatible}
\usepackage[toc,page]{appendix}
\newcommand{\iu}{\mathrm{i}\mkern1mu}
\usepackage{lineno}

\newtheorem{theorem}{Theorem}
\newtheorem{lemma}{Lemma}
\newtheorem{definition}{Definition}
\newtheorem{corollary}{Corollary}
\newtheorem{proposition}{Proposition}

\numberwithin{equation}{section}

\DeclareMathOperator{\sign}{sgn}
\DeclareMathOperator{\lip}{Lip}
\DeclareMathOperator{\diam}{diam}

\usepackage{mathptmx}      %
\begin{document}

\title{Sharp Bounds on the Approximation Rates, Metric Entropy, and $n$-widths of Shallow Neural Networks
}

\author{Jonathan W. Siegel \\
  Department of Mathematics\\
  Pennsylvania State University\\
  University Park, PA 16802 \\
  \texttt{jus1949@psu.edu} \\
  \And Jinchao Xu \\
  Department of Mathematics\\
  Pennsylvania State University\\
  University Park, PA 16802 \\
  \texttt{jxx1@psu.edu} \\
}

\maketitle

\begin{abstract}
In this article, we study approximation properties of the variation spaces corresponding to shallow neural networks with a variety of activation functions. We introduce two main tools for estimating the metric entropy, approximation rates, and $n$-widths of these spaces. First, we introduce the notion of a smoothly parameterized dictionary and give upper bounds on the non-linear approximation rates, metric entropy and $n$-widths of their absolute convex hull. The upper bounds depend upon the order of smoothness of the parameterization. This result is applied to dictionaries of ridge functions corresponding to shallow neural networks, and they improve upon existing results in many cases. Next, we provide a method for lower bounding the metric entropy and $n$-widths of variation spaces which contain certain classes of ridge functions. This result gives sharp lower bounds on the $L^2$-approximation rates, metric entropy, and $n$-widths for variation spaces corresponding to neural networks with a range of important activation functions, including ReLU$^k$ activation functions and sigmoidal activation functions with bounded variation.
\end{abstract}

\section{Introduction}
\subsection{Preliminaries}
The class of shallow neural networks with activation function $\sigma:\mathbb{R}\rightarrow \mathbb{R}$ is a popular function class used in supervised learning algorithms. This class of functions on $\mathbb{R}^d$ is given by
\begin{equation}
 \Sigma^d_n(\sigma) = \left\{\sum_{i=1}^n \sigma(\omega_i\cdot x + b_i):~\omega_i\in \mathbb{R}^d,~b_i\in \mathbb{R}\right\},
\end{equation}
where $\sigma$ is an activation function and $n$ is the width of the network. There is a rich literature on the approximation properties and statistical inference from this class of functions \cite{jones1992simple,barron1993universal,klusowski2018approximation,lee1996efficient,bach2017breaking}, with a special focus on the case when $\sigma$ is a sigmoidal activation function or when $\sigma = \max(0,x)^k$ is a power of the rectified linear unit.

In this work, we consider the approximation properties of shallow neural networks from the point of view of non-linear dictionary approximation \cite{devore1998nonlinear}. Let $X$ be Banach space and $\mathbb{D}\subset X$ be a uniformly bounded dictionary, i.e. $\mathbb{D}$ is a subset such that $\sup_{d\in \mathbb{D}} \|d\|_X = K_\mathbb{D} < \infty$. Non-linear dictionary approximation considers approximating a target function $f$ by non-linear $n$-term dictionary expansions, i.e. by an element of the set
\begin{equation}\label{non-linear-dictionary-class}
 \Sigma_{n}(\mathbb{D}) = \left\{\sum_{j=1}^n a_jh_j:~h_j\in \mathbb{D}\right\}.
\end{equation}
The approximation is non-linear since the elements $h_j$ in the expansion will depend upon the target function $f$. It is often also important to have some control over the coefficients $a_j$ which occur in the expansion \eqref{non-linear-dictionary-class}. For this purpose, we introduce the sets
\begin{equation}\label{dictionary-expansion-definition}
\Sigma_{n,M}(\mathbb{D}) = \left\{\sum_{j=1}^n a_jh_j:~h_j\in \mathbb{D},~\sum_{i=1}^n|a_i|\leq M\right\}~~\text{and}~~\Sigma_{n,M}^\infty(\mathbb{D}) = \left\{\sum_{i=1}^n a_i g_i,~g_i\in \mathbb{D},~\max_{i=1,...,n} |a_i| \leq M\right\},
\end{equation}
 which correspond to coefficients which are bounded in $\ell^1$ and $\ell^\infty$, respectively.

The application of this framework to shallow neural networks comes by considering the dictionary
\begin{equation}\label{sigma-dictionary-definition}
 \mathbb{D}^d_\sigma = \{\sigma(\omega\cdot x + b):~\omega\in \mathbb{R}^d,~b\in \mathbb{R}\}\subset L^p(\Omega),
\end{equation}
where $\sigma$ is an appropriate activation function and $\Omega\subset \mathbb{R}^d$ is a bounded domain. For this dictionary, the set
\begin{equation}
 \Sigma_{n}(\mathbb{D}^d_\sigma) = \Sigma^d_n(\sigma) = \left\{\sum_{j=1}^n a_j\sigma(\omega_j \cdot x + b_j):~\omega_j\in \mathbb{R}^d,~b_j\in \mathbb{R}\right\}
\end{equation}
is exactly the set of shallow neural networks with activation function $\sigma$ and width $n$. Note that we are suppressing the dependence on the underlying space $L^p(\Omega)$ for notational simplicity.

For activation functions $\sigma$ which are bounded, the dictionary $\mathbb{D}^d_\sigma$ is uniformly bounded in $L^p(\Omega)$. This holds for the class of sigmoidal activation functions, i.e. activation functions which satisfy $\lim_{x\rightarrow -\infty} \sigma(x) = 0$ and $\lim_{x\rightarrow \infty} \sigma(x) = 1$, for example. In this case we will consider the dictionary $\mathbb{D}_\sigma^d$ given in \eqref{sigma-dictionary-definition}. 

However, when the activation function $\sigma$ is not bounded, the dictionary $\mathbb{D}_\sigma^d$ will not in general be uniformly bounded in $L^p(\Omega)$. This is the case for the important activation functions $\sigma_k(x) = \text{ReLU}^k(x) := \max(x,0)^k$ when $k > 0$, for instance. In this case, we modify the definition \eqref{sigma-dictionary-definition} appropriately in order to guarantee uniform boundedness of the dictionary. When $\sigma_k(x) = \text{ReLU}^k(x)$ (here when $k=0$, we interpret $\sigma_k(x)$ to be the Heaviside function) we constrain the weights $\omega$ and $b$ and consider the dictionary
\begin{equation}\label{relu-k-space-definition}
 \mathbb{P}^d_k = \{\sigma_k(\omega\cdot x + b):~\omega\in S^{d-1},~b\in [c_1,c_2]\}\subset L^p(\Omega),
\end{equation}
where $S^{d-1} = \{\omega\in \mathbb{R}^d:~|\omega| = 1\}$ is the unit sphere and the constants $c_1$ and $c_2$ are chosen to satisfy
\begin{equation}
    c_1 < \inf\{x\cdot\omega,~\omega\in S^{d-1},~x\in \Omega\} < \sup\{x\cdot\omega,~\omega\in S^{d-1},~x\in \Omega\} < c_2.
\end{equation}
We remark that any choice of $c_1$ and $c_2$ satisfying the above conditions leads to a dictionary which is equivalent up to polynomials (see \cite{siegel2021characterization}). Due to the homogeneity of the activation function $\sigma_k$, the set of non-linear dictionary expansions $\Sigma_n(\mathbb{P}_k^d)$ coincides with the set of shallow ReLU$^k$ neural networks with width $n$.

An important model class of functions which can be efficiently approximated by non-linear dictionary expansions is given by the variation norm of the dictionary $\mathbb{D}$ \cite{devore1998nonlinear,barron2008approximation,barron1993universal,bach2017breaking,kurkova2001bounds,kurkova2002comparison}. Consider the set
\begin{equation}\label{unit-ball-definition}
 B_1(\mathbb{D}) = \overline{\left\{\sum_{j=1}^n a_jh_j:~n\in \mathbb{N},~h_j\in \mathbb{D},~\sum_{i=1}^n|a_i|\leq 1\right\}},
\end{equation}
which is the closure of the convex, symmetric hull of $\mathbb{D}$. Using this set we define a norm, $\|\cdot\|_{\mathcal{K}_1(\mathbb{D})}$, on $X$ given by the gauge (see for instance \cite{rockafellar1970convex}) of $B_1(\mathbb{D})$,
\begin{equation}\label{norm-definition}
 \|f\|_{\mathcal{K}_1(\mathbb{D})} = \inf\{c > 0:~f\in cB_1(\mathbb{D})\}.
\end{equation}
This norm is defined so that $B_1(\mathbb{D})$ is the unit ball of $\|\cdot\|_{\mathcal{K}_1(\mathbb{D})}$. We also consider the subspace of $X$ defined by the variation norm, which we denote
\begin{equation}\label{space-definition}
\mathcal{K}_1(\mathbb{D}) := \{f\in X:~\|f\|_{\mathcal{K}_1(\mathbb{D})} < \infty\}.
\end{equation}
As long as $\sup_{d\in \mathbb{D}} \|d\|_X = K_{\mathbb{D}} < \infty$, the space $\mathcal{K}_1(\mathbb{D})$ is a Banach space (see \cite{siegel2021characterization} for instance). The space $\mathcal{K}_1(\mathbb{D})$ is typically called the variation space and the norm $\|\cdot\|_{\mathcal{K}_1(\mathbb{D})}$ is typically called the variation norm or the atomic norm with respect to the dictionary $\mathbb{D}$. 

As shown in \cite{siegel2021characterization}, for the dictionary $\mathbb{P}_1^d$ corresponding to shallow neural networks with ReLU activation function, the space $\mathcal{K}_1(\mathbb{P}_1^d)$ is equivalent to the Barron space introduced and studied in \cite{weinan2019barron,wojtowytsch2020representation}. Further, the space $\mathcal{K}_1(\mathbb{P}_k^d)$ for general $k$ can be characterized in terms of the Radon transform and is equivalent to the Radon BV space introduced in the context of shallow neural networks in \cite{parhi2020banach,parhi2021kinds,ongie2019function}.

The first major problem we consider in this work is how efficiently functions from the variation space $\mathcal{K}_1(\mathbb{D})$ can be approximated by non-linear dictionary expansions from the sets $\Sigma_n(\mathbb{D})$, $\Sigma_{n,M}(\mathbb{D})$, or $\Sigma_{n,M}^\infty(\mathbb{D})$. There is a rich literature on this problem, which has significant applications to statistics and machine learning (see, for instance \cite{pisier1981remarques,jones1992simple,barron1993universal,klusowski2018approximation,bach2017breaking,makovoz1996random,devore1998nonlinear,kurkova2001bounds,barron2008approximation}), and the dictionaries $\mathbb{D}_\sigma^d$ and $\mathbb{P}_k^d$ corresponding to shallow neural networks are of particular interest.

There is a close relationship between this problem and fundamental approximation theoretic quantities such as the metric entropy and $n$-widths of the convex hull $B_1(\mathbb{D})$ (which we recall is the unit ball of $\mathcal{K}_1(\mathbb{D})$).  Let us give a brief overview of these notions, for more details, see for instance \cite{lorentz1996constructive,pinkus2012n,lang2011eigenvalues,triebel1995interpolation}. We remark that the notions of entropy and $n$-widths can also be naturally defined for operators $T:X\rightarrow Y$ between two Banach spaces \cite{pietsch1974s,pietsch1978operator}. Here we will only discuss these notions for subsets of a Banach space for simplicity.

The notion of metric entropy was first introduced by Kolmogorov \cite{kolmogorov1958linear}. The (dyadic) entropy numbers $\epsilon_n(A)$ of a set $A\subset X$ are defined by
\begin{equation}
 \epsilon_n(A)_X = \inf\{\epsilon > 0:~\text{$A$ is covered by $2^n$ balls of radius $\epsilon$}\}.
\end{equation}
Roughly speaking, the entropy numbers indicate how precisely we can specify elements of $A$ given $n$ bits of information and are closely related to approximation by stable non-linear methods \cite{cohen2020optimal}.

The Kolmogorov $n$-widths of a set $A\subset X$ measure how accurately the set $A$ can be approximated by linear subspaces and are given by
\begin{equation}
 d_n(A)_X = \inf_{Y_n}\sup_{x\in A}\inf_{y\in Y_n}\|x - y\|_X,
\end{equation}
where the first infimum is over the collection of subspaces $Y_n$ of dimension $n$. The Gelfand $n$-widths, which are important in compressed sensing \cite{donoho2006compressed}, measure how accurately elements from $A$ can be recovered from linear measurements and are given by 
\begin{equation}
 d^n(A)_X = \inf_{U_n}\sup\{\|x\|_H:~x\in U_n\cap A\},
\end{equation}
where the infemum is taken over all closed subspaces $U_n$ of codimension $n$. The linear $n$-widths are closely related to the Kolmogorov $n$-widths but require the approximation to be given by a linear map and are defined by
\begin{equation}
 \delta_n(A)_X = \inf_{T_n}\sup_{x\in A} \|x - T_n(x)\|_X,
\end{equation}
where the infemum is taken over all linear operators of rank $n$. In a Hilbert space the linear $n$-widths and Kolmogorov $n$-widths coincide since $T_n$ can be taken as the orthogonal projection operator. Finally, the Bernstein widths are given by
\begin{equation}
 b_n(A)_X = \sup_{X_{n}} \inf_{x\in \partial(A\cap X_{n})} \|x\|_X,
\end{equation}
where the supremum is taken over all subspaces $X_{n}\subset X$ of dimension $n$. The Bernstein widths give a lower bound on the best possible continuous non-linear approximation of the set $A$ using $n$-parameters \cite{devore1989optimal}.

The second main problem we consider is the determination of the asymptotics of the metric entropy and the different types of $n$-widths mentioned above for the class $B_1(\mathbb{D})\subset X$ for different dictionaries $\mathbb{D}\subset X$. This problem has been studied in functional analysis, probability theory and approximation theory \cite{ball1990entropy,carl1981entropy,carl1988gelfand,carl1999metric,carl2013gelfand,carl2014entropy,dudley1967sizes,dudley1987universal}, and has important applications to statistical learning theory. For example, the asymptotics of the metric entropy can be used to determine the minimax rates of convergence for statistical estimators on a class of functions \cite{yang1999information}.

\subsection{Prior Results}
Let us begin with results concerning the approximation of the convex hull $B_1(\mathbb{D})$ by non-linear dictionary expansions. A classical result of Maurey \cite{pisier1981remarques} (see also \cite{jones1992simple,barron1993universal,devore1998nonlinear}) states that if $X$ is a type-2 Banach space, which includes all Hilbert spaces and $L^p$ for $2\leq p < \infty$,
then for functions $f\in B_1(\mathbb{D})$ we have the approximation rate
\begin{equation}\label{fundamental-bound}
 \inf_{f_n\in \Sigma_{n,1}(\mathbb{D})} \|f - f_n\|_X \lesssim K_\mathbb{D}n^{-\frac{1}{2}},
\end{equation}
where the suppressed constant depends only upon the space $X$. An equivalent formulation of this result, which is sometimes more convenient is that for $f\in \mathcal{K}_1(\mathbb{D})$ we have
\begin{equation}
 \inf_{f_n\in \Sigma_{n,M}(\mathbb{D})} \|f - f_n\|_X \lesssim K_\mathbb{D}\|f\|_{\mathcal{K}_1(\mathbb{D})}n^{-\frac{1}{2}},
\end{equation}
where the bound $M$ can be taken as $M = \|f\|_{\mathcal{K}_1(\mathbb{D})}$.
In Section \ref{sampling-argument-sec}, we give the precise definition of a type-2 Banach space and prove \eqref{fundamental-bound} using Maurey's sampling argument, which is a fundamental building block of the theory.

Applying this result to the dictionaries $\mathbb{D}_\sigma^d$ and $\mathbb{P}_k^d$ immediately yields the following dimension independent approximation rates in $L^p$ for $2\leq p < \infty$:
\begin{equation}
    \inf_{f_n\in \Sigma_{n}^d(\sigma)} \|f - f_n\|_{L^p(\Omega)} \lesssim \|f\|_{\mathcal{K}_1(\mathbb{D}_\sigma^d)}n^{-\frac{1}{2}}
\end{equation}
for shallow neural networks with bounded activation function $\sigma$ on the variation space $\mathcal{K}_1(\mathbb{D}_\sigma^d)$, and
\begin{equation}\label{relu-k-fundamental-approximation-rate}
    \inf_{f_n\in \Sigma_{n}^d(\sigma_k)} \|f - f_n\|_{L^p(\Omega)} \lesssim \|f\|_{\mathcal{K}_1(\mathbb{P}_k^d)}n^{-\frac{1}{2}}
\end{equation}
for shallow ReLU$^k$ neural networks on the variation space $\mathcal{K}_1(\mathbb{P}_k^d)$. These results were first obtained in the $L^2$-norm by Jones \cite{jones1992simple} for the activation function $\sigma(x) = \cos(x)$ and by Barron \cite{barron1993universal} for sigmoidal activation functions. To be precise, Barron obtained approximation rates in terms of the spectral Barron semi-norm
\begin{equation}\label{spectral-barron-semi}
 |f|_{\mathcal{B}_s(\Omega)} := \inf_{f_e|_{\Omega} = f} \int_{\mathbb{R}^d} |\hat{f}(\xi)||\xi|^sd\xi,
\end{equation}
where the infimum is over all extensions $f_e\in L^1(\mathbb{R}^d)$ of $f$. Barron showed that modulo constants the semi-norm \eqref{spectral-barron-semi} for $s=1$ controls the variation norm $\mathcal{K}_1(\mathbb{D}_\sigma^d)$ for any sigmoidal activation function $\sigma$, from which his seminal result
    \begin{equation}
    \inf_{f_n\in \Sigma_{n}^d(\sigma)} \|f - f_n\|_{L^2(\Omega)} \lesssim |f|_{\mathcal{B}_1(\Omega)}n^{-\frac{1}{2}}
\end{equation}
follows \cite{barron1993universal}. 

It will be more convenient for us in what follows to consider the spectral Barron norm instead of the semi-norm \eqref{spectral-barron-semi}, which is given by \cite{siegel2020approximation}
\begin{equation}\label{spectral-barron-norm}
 \|f\|_{\mathcal{B}_s(\Omega)} := \inf_{f_e|_{\Omega} = f} \int_{\mathbb{R}^d} |\hat{f}(\xi)|(1 + |\xi|)^sd\xi,
\end{equation}
and the corresponding spectral Barron space $\mathcal{B}_s$.
Barron's results have been generalized to ReLU$^k$ activation functions in \cite{klusowski2018approximation,CiCP-28-1707}. It is shown that 
\begin{equation}\label{barron-spectral-barron-bound}
    \|f\|_{\mathcal{K}_1(\mathbb{P}_k^d)} \lesssim \|f\|_{\mathcal{B}_{k+1}(\Omega)},
\end{equation}
from which it follows that the approximation rate \eqref{relu-k-fundamental-approximation-rate} applies to the spectral Barron space $\mathcal{B}_{k+1}(\Omega)$. These results have also been extended to a very general class of activation functions in \cite{siegel2020approximation}. Interestingly, \eqref{barron-spectral-barron-bound} is not an equivalence and quantifying the gap between $\mathcal{K}_1(\mathbb{P}_k^d)$ and $\mathcal{B}_{k+1}(\Omega)$ is an open problem.

An important result, first observed by Makovoz \cite{makovoz1996random}, is that for certain dictionaries the rate in \eqref{fundamental-bound} can be improved. In particular, for the dictionary $\mathbb{D}^d_\sigma$ corresponding to neural networks with certain sigmoidal activation functions, Makovoz obtained the approximation rate
\begin{equation}\label{makovoz-original}
 \inf_{f_n\in \Sigma_{n,M}(\mathbb{D}^d_\sigma)} \|f - f_n\|_{L^2(\Omega)} \lesssim n^{-\frac{1}{2}-\frac{1}{2d}}
\end{equation}
for $f\in B_1(\mathbb{D}^d_\sigma)$
(Note that here and in what follows the implied constant is independent of $n$, and the $\ell^1$-bound $M$ is fixed and independent of $n$ as well.)
Furthermore, improved rates have been obtained for other dictionaries. In particular, in \cite{klusowski2018approximation}, the dictionaries $\mathbb{P}^d_k$ corresponding to neural networks with activation function $\sigma = [\max(0,x)]^k$ are studied for $k=1,2$ and it is shown that for $f\in B_1(\mathbb{P}^d_k)$
\begin{equation}\label{relu-k-original-rate}
 \inf_{f_n\in \Sigma_{n,M}(\mathbb{P}^d_k)} \|f - f_n\|_{L^2(\Omega)} \lesssim n^{-\frac{1}{2}-\frac{1}{d}}.
\end{equation}
We remark that more generally, it is shown that these rates hold in $L^\infty$ up to logarithmic factors.
This analysis is extended to $k\geq 3$ in \cite{CiCP-28-1707}, where the same approximation rate is attained. 

In addition, when $k=1$, i.e. for the ReLU activation function, the rate \eqref{relu-k-original-rate} can be improved to
\begin{equation}
 \inf_{f_n\in \Sigma_{n,M}(\mathbb{P}^d_1)} \|f - f_n\|_{L^2(\Omega)} \lesssim n^{-\frac{1}{2}-\frac{3}{2d}},
\end{equation}
and this result holds also in $L^\infty$ \cite{bach2017breaking}. This result is proved using a different set of combinatorial arguments from geometric discrepancy theory \cite{matouvsek1996improved}. 

These results raise the natural question of what the optimal approximation rates for $\Sigma_{n,M}(\mathbb{P}^d_k)$ on the set $B_1(\mathbb{P}^d_k)$ are.
Specifically, for each $k=0,1,2,...$ and dimension $d=2,...$ (the case $d=1$ is comparatively trivial), what is the largest possible value of $\alpha := \alpha(k,d)$ such that for $f\in B_1(\mathbb{P}^d_k)$ we have
\begin{equation}\label{optimal-approx-equation}
 \inf_{f_n\in \Sigma_{n,M}(\mathbb{P}^d_k)} \|f - f_n\|_{L^2(\Omega)} \lesssim n^{-\frac{1}{2}-\alpha(k,d)}.
\end{equation}
The results above imply that $\alpha(k,d) \geq \frac{1}{2d}$ for $k=0$, $\alpha(k,d) \geq \frac{3}{2d}$ for $k = 1$, and $\alpha(k,d) \geq \frac{1}{d}$ for $k > 1$. When $d > 1$, the best available upper bounds on $\alpha(k,d)$ are $\alpha(k,d) \leq \frac{k+1}{d}$ (see \cite{makovoz1996random,klusowski2018approximation}), except in the case $k=0$, $d=2$, where Makovoz obtains the sharp bound $\alpha(0,2) = \frac{1}{4}$ \cite{makovoz1996random}. 

\subsection{Main Results}

Our results can be divided into two categories. First, we consider upper bounds on approximation rates, entropy, and $n$-widths. Previous approximation rates, such as those obtained in \cite{makovoz1996random,klusowski2018approximation,CiCP-28-1707}, and entropy bounds on the convex hull $B_1(\mathbb{D})$, such as those obtained in \cite{carl1997metric,carl1999metric}, rely upon covering the dictionary $\mathbb{D}$ by balls of radius $\epsilon$ and using a stratified sampling argument.

In order to improve upon these results, the key idea is to use the smoothness of the dictionary $\mathbb{P}_k^d$. We obtain a higher-order generalization of these results by introducing the notion of a smoothly parameterized dictionary. A dictionary $\mathbb{D}$ is smoothly parameterized to order $s$ if there exists a parameterization map $\mathcal{P}:\mathcal{M}\rightarrow X$, where $\mathcal{M}$ is a smooth manifold, such that $\mathbb{D} \subset \mathcal{P}(\mathcal{M})$, and such that $\mathcal{P}$ is smooth to order $s$ in an appropriate sense. We give a detailed definition in Section \ref{main-result-1-section}, where we show that if $\mathbb{D}$ is smoothly parameterized to order $s$ by a compact $d$-dimensional manifold $\mathcal{M}$ and $X$ is a type-$2$ Banach space, then
\begin{equation}\label{smooth-approximation-rate}
 \inf_{f_n\in \Sigma_{n,M}(\mathbb{D})} \|f - f_n\|_{X} \lesssim n^{-\frac{1}{2}-\frac{s}{d}},
\end{equation}
for some $M < \infty$. Here the implied constant depends upon the manifold $\mathcal{M}$, the parameterization $\mathcal{P}$, and the type-$2$ constant of $X$. We apply this result to the dictionaries $\mathbb{P}_k^d\subset L^p(\Omega)$, which we show are smoothly parameterized to order $k + \frac{1}{p}$ by the compact, $d$-dimensional manifold $S^{d-1}\times [c_1,c_2]$. This allows us to improve upon the approximation rates for ReLU$^k$ networks described above when $k\geq 2$, and in particular show that $\alpha(k,d) \geq \frac{2k+1}{2d}$ for all $k\geq 0$.

We also give upper bounds on the metric entropy and $n$-widths of the convex hull $B_1(\mathbb{D})$ when $\mathbb{D}$ is a smoothly parameterized dictionary. In particular, in Section \ref{main-result-1-section} we show that if $\mathbb{D}$ is smoothly parameterized to order $s$ by a compact, $d$-dimensional manifold, then we have
\begin{equation}
 d_n(B_1(\mathbb{D}))_X\lesssim n^{-\frac{s}{d}}.
\end{equation}
If in addition the space $X$ is a type-$2$ Banach space, then we have the bound
\begin{equation}
 \epsilon_n(B_1(\mathbb{D}))_X\lesssim n^{-\frac{1}{2}-\frac{s}{d}}.
\end{equation}
Finally, if $X$ is a Hilbert space, we obtain an analogous bound on the Gelfand numbers of the convex hull of $\mathbb{D}$. The Gelfand numbers of a convex hull are introduced and studied in \cite{carl1999metric,carl2013gelfand}. While closely related to the Gelfand widths $d^n(B_1(\mathbb{D}))$, they are subtly different \cite{edmunds2013gelfand}. We give the precise definition of the Gelfand numbers and provide an example illustrating this difference in Section \ref{gelfand-numbers-subsection}.

These results generalize the results from \cite{carl1997metric,carl1999metric,carl2013gelfand,carl2014entropy} by taking into account the smoothness of the dictionary in addition to the compactness. The application to $\mathbb{P}_k^d$ gives the bound
\begin{equation}
 d_n(B_1(\mathbb{P}_k^d))_{L^p(\Omega)} \lesssim n^{-\frac{pk+1}{pd}}
\end{equation}
on the Kolmogorov $n$-widths for $1 < p < \infty$, and the bound
\begin{equation}\label{introduction-entropy-upper-bound}
 \epsilon_n(B_1(\mathbb{P}_k^d))_{L^p(\Omega)} \lesssim n^{-\frac{1}{2}-\frac{pk+1}{pd}}
\end{equation}
on the metric entropy for $2\leq p < \infty$. Based on the $L^\infty$ approximation rates which can be derived using geometric discrepancy theory \cite{bach2017breaking,matouvsek1996improved} in the case $k=1$, we only expect the entropy upper bound to be sharp for $p = 2$, however.

Next, we consider lower bounds on approximation rates, metric entropy, and $n$-widths of the spaces $B_1(\mathbb{D}_\sigma^d)$ and $B_1(\mathbb{P}_k^d)$. The key here is a generalization of a construction of Makovoz \cite{makovoz1996random}, which enables us to find a large collection of nearly orthogonal functions in $B_1(\mathbb{D})$ when the dictionary $\mathbb{D}$ contains a certain class of ridge functions. As a consequence of this construction, we obtain the following lower bounds on the entropy and $n$-widths in $L^2$
\begin{equation}
 d_n(B_1(\mathbb{P}_k^d))_{L^2(\Omega)} \gtrsim n^{-\frac{2k+1}{2d}},~\epsilon_n(B_1(\mathbb{P}_k^d))_{L^2(\Omega)} \gtrsim n^{-\frac{1}{2}-\frac{2k+1}{2d}},~b_n(B_1(\mathbb{P}_k^d))_{L^2(\Omega)}\gtrsim n^{-\frac{1}{2}-\frac{2k+1}{2d}}.
\end{equation}
This method is quite general, and one can also obtain lower bounds for $B_1(\mathbb{D}_\sigma^d)$ for a variety of other activation functions $\sigma$, but we do not pursue this further here. Note that this lower bound combined with the upper bound \eqref{introduction-entropy-upper-bound} for $p=2$ gives the sharp rate of decay for the metric entropy of $B_1(\mathbb{P}_k^d)$
\begin{equation}\label{introduction-entropy-decay-rate}
    \epsilon_n(B_1(\mathbb{P}_k^d))_{L^2(\Omega)} \eqsim n^{-\frac{1}{2}-\frac{2k+1}{2d}}.
\end{equation}
In addition the lower bounds on the entropy enable us to obtain the sharp upper bound $\alpha(k,d)\leq \frac{2k+1}{2d}$ on the exponent of approximation by shallow ReLU$^k$ networks. Further, we use these lower bounds to show that the exponent in the approximation rates \eqref{optimal-approx-equation} cannot be improved even when relaxing the $\ell^1$-norm constraint on the weights to an $\ell^\infty$-norm constraint, i.e. by using the set $\Sigma_{n,M}^\infty(\mathbb{P}_k^d)$ instead of $\Sigma_{n,M}(\mathbb{P}_k^d)$. 

Making use of a technical lemma proved in Section \ref{consequences-section}, these lower bounds also carry over to sigmoidal activation functions $\sigma$ which have bounded variation. This generalizes previous results, which require more stringent assumptions on the activation function $\sigma$ in order to obtain lower bounds \cite{makovoz1998uniform}.

The entropy lower bounds also quantify the gap between the spaces $\mathcal{K}_1(\mathbb{P}_k^d)$ and $\mathcal{B}_{k+1}(\Omega)$. Specifically, the unit ball in the spectral Barron space $\mathcal{B}_{k+1}(\Omega)$ satisfies
\begin{equation}
    \epsilon_{n\log n}(\{f:\|f\|_{\mathcal{B}_{k+1}(\Omega)} \leq 1\})_{L^2(\Omega)} \lesssim n^{-\frac{1}{2}-\frac{k+1}{d}}.
\end{equation}
This follows from the approximation results obtained in \cite{siegel2022high} combined with Theorem \ref{dictionary-carls-theorem}. Conversely, the metric entropy of $B_1(\mathbb{P}_k^d)$ decays at the rate \eqref{introduction-entropy-decay-rate}, which is slower by a factor of $n^{1/2d}$.

Finally, the lower bounds on the metric entropy, Kolmogorov, and Bernstein $n$-widths show that linear methods are dramatically worse that shallow neural networks, and even non-linear methods cannot improve upon shallow neural networks on the class $B_1(\mathbb{P}_k^d)$ if either stability \cite{cohen2020optimal} or continuity \cite{devore1989optimal} of the approximation method is required.

The paper is organized as follows. In Section \ref{sampling-argument-sec}, we review the Barron-Maurey-Jones sampling argument for type-2 Banach spaces. In Section \ref{main-result-1-section}, we study smoothly parameterized dictionaries in detail and derive approximation rates for the set $B_1(\mathbb{D})$ when $\mathbb{D}$ is smoothly parameterized by a compact manifold. Here we also extend existing methods to obtain upper bounds on the entropy and $n$-widths of $B_1(\mathbb{D})$ for such dictionaries $\mathbb{D}$. We apply these results to obtain an upper bound on the approximation rates, entropy numbers and $n$-widths of $B_1(\mathbb{P}^d_k)$. Next, in Section \ref{lower-bounds-section}, we prove lower bounds on the approximation rates, metric entropy, and $n$-widths for dictionaries of ridge functions. Finally, we give some concluding remarks and further research directions.

\section{Type-2 Spaces and Maurey's Sampling Argument}\label{sampling-argument-sec}
In this section, we give the definition and basic properties of type-$2$ Banach spaces and prove the approximation rate \eqref{fundamental-bound} using a sampling argument due to Maurey \cite{pisier1981remarques}. This sampling argument was first used in the context of neural network approximation by Barron \cite{barron1993universal}. We refer to \cite{ledoux2013probability}, Chapter 9 and \cite{albiac2006topics}, Chapter 6 for a detailed theory of type-2 spaces (and the more general type-p spaces which we will not discuss further here).

\begin{definition}\label{type-2-definition}
 A Banach space $X$ is a type-2 Banach space if there exists a constant $C_{2,X} < \infty$ such that for any $n \geq 1$ and $f_1,...,f_n\in X$ we have
 \begin{equation}\label{type-2-definition}
 \left(\mathbb{E} \left\|\sum_{i=1}^n\epsilon_if_i\right\|^2_X\right)^{1/2} \leq C_{2,X}\left(\sum_{i=1}^n\|f_i\|_X^2\right)^{1/2},
\end{equation}
where the expectation is taken over independent Rademacher random variables $\epsilon_1,...,\epsilon_n$, i.e. $$\mathbb{P}(\epsilon_i = 1) = \mathbb{P}(\epsilon_i = -1) = \frac{1}{2}.$$
The constant $C_{2,X}$ is called the type-2 constant of the space $X$.
\end{definition}
It is clear that a Hilbert space $X$ is a type-2 space with type-2 constant $C_{2,X} = 1$, since by expanding the inner product we have
\begin{equation}
    \mathbb{E} \left\|\sum_{i=1}^n\epsilon_if_i\right\|^2_X = \sum_{i=1}^n \sum_{j=1}^n \mathbb{E}(\epsilon_i\epsilon_j)\langle f_i,f_j\rangle_X = \sum_{i=1}^n\|f_i\|_X^2.
\end{equation}
Here independence of the Rademacher variables implies that $\mathbb{E}(\epsilon_i\epsilon_j) = \delta_{ij}$.

We remark that by the Khintchine-Kahane Theorem \cite{kahane1964sommes}, the condition \eqref{type-2-definition} is equivalent to
\begin{equation}\label{kahane-type-2}
    \left(\mathbb{E} \left\|\sum_{i=1}^n\epsilon_if_i\right\|^p_X\right)^{1/p} \leq C_{2,X}\left(\sum_{i=1}^n\|f_i\|_X^2\right)^{1/2}
\end{equation}
for any $1\leq p < \infty$, up to a change in the constant $C_{2,X}$. 

Let $(X,\mathcal{A},\mu)$ be a measure space. We also have that $L^p(d\mu)$ for $2\leq p < \infty$ is a type-2 Banach space. This follows from Khintchine's inequality, which states that for $0 < p < \infty$ there exists a constant $C_p < \infty$ such that for $n\geq 1$ and real numbers $x_1,...,x_n$ we have
\begin{equation}
    \left(\mathbb{E} \left|\sum_{i=1}^n\epsilon_ix_i\right|^p\right)^{\frac{1}{p}} \leq C_p\left(\sum_{i=1}^n |x_i|^2\right)^{\frac{1}{2}}.
\end{equation}
Letting $f_1,...,f_n\in L^p(\mu)$ and interchanging the expectation and integration, we calculate
\begin{equation}
    \mathbb{E} \left\|\sum_{i=1}^n\epsilon_if_i\right\|^p_{L^p(d\mu)} = \int_{X}\mathbb{E} \left|\sum_{i=1}^n\epsilon_if_i(x)\right|^pd\mu(x) \leq C_p^p\int_X\left(\sum_{i=1}^nf_i(x)^2\right)^{p/2}d\mu(x) = C_p^p\left\|\sum_{i=1}^nf_i^2\right\|_{L^{p/2}(d\mu)}^{p/2}.
\end{equation}
Using the triangle inequality (valid since $p \geq 2$), we get
\begin{equation}
    \mathbb{E} \left\|\sum_{i=1}^n\epsilon_if_i\right\|^p_{L^p(d\mu)} \leq C^p_p\left\|\sum_{i=1}^nf_i^2\right\|_{L^{p/2}(d\mu)}^{p/2} \leq C^p_p\left(\sum_{i=1}^n\left\|f_i^2\right\|_{L^{p/2}(d\mu)}\right)^{p/2} = C_p^p\left(\sum_{i=1}^n\left\|f_i\right\|_{L^{p}(d\mu)}^2\right)^{p/2}.
\end{equation}
This proves \eqref{kahane-type-2} for $L^p(d\mu)$, so that $L^p(d\mu)$ is a type-2 Banach space. We remark that the optimal constant $C_p$ in Khintchine's inequality is known \cite{haagerup1981best} and scales like $\sqrt{p}$.

Finally, we prove the approximation rate \eqref{fundamental-bound} in type-2 Banach spaces, which is originally due to Maurey \cite{pisier1981remarques}.
\begin{theorem}\label{maurey-sampling-argument-thm}
 Suppose that $X$ is a type-2 Banach space and $\mathbb{D}\subset X$ is a dictionary with $K_\mathbb{D}:=\sup_{d\in \mathbb{D}}\|d\|_X < \infty$. Then for $f\in B_1(\mathbb{D})$, we have
 \begin{equation}
     \inf_{f_n\in \Sigma_{n,1}(\mathbb{D})} \|f - f_n\|_X \leq 4C_{2,X}K_\mathbb{D}n^{-\frac{1}{2}}.
 \end{equation}
\end{theorem}
\begin{proof}
 Let $\delta > 0$ be arbitrary. Since $f\in B_1(\mathbb{D})$, for some integer $N = N(\delta)$, there exists an $f_\delta$ of the form
 \begin{equation}
     f_\delta = \sum_{i=1}^Na_id_i
 \end{equation}
 with $d_i\in \mathbb{D}$ and $\sum_{i=1}^N |a_i| = 1$ such that $\|f - f_\delta\|_X < \delta$. We will show that there exists an $f_n\in \Sigma_{n,1}$ with $$\|f_\delta - f_n\|_X \leq 4C_{2,X}K_\mathbb{D}n^{-1/2}$$ which completes the proof since $\delta > 0$ was arbitrary. To this end, define a random variable $Y$ with values in $X$ by (recall that $\sum_{i=1}^N |a_i| = 1$)
 \begin{equation}
     \mathbb{P}(Y = \sign(a_i)d_i) = |a_i|.
 \end{equation}
 Note that by construction we have $\mathbb{E}(Y) = f_\delta$. Let $Y_1,...,Y_n$ be independent copies of $Y$ and consider the empirical average
 \begin{equation}
     Z_n = \frac{1}{n}\sum_{i=1}^n Y_i.
 \end{equation}
 We will show that
 \begin{equation}\label{expectation-difference-bound}
     \mathbb{E}\|Z_n - f_\delta\|^2_X = \mathbb{E}\left\|\frac{1}{n}\sum_{i=1}^n (Y_n - f_\delta)\right\|_X^2\leq 16C_{2,X}^2K_\mathbb{D}^2n^{-1}.
 \end{equation}
 This implies that there must be a realization of the random variable $Z_n$, i.e. a value $Z_n(\omega)$ for some $\omega\in \Omega$ in the underlying probability space, such that $\|Z_n(\omega) - f_\delta\|_X\leq 2C_{2,X}K_\mathbb{D}n^{-1/2}$. Since the values of the random variable $Z_n$ lie in $\Sigma_{n,1}(\mathbb{D})$ by construction, this completes the proof.
 
 To prove \eqref{expectation-difference-bound} we will show that if a sequence of i.i.d. random variables $R_1,...,R_n$ with values in $X$ satisfies $\mathbb{E}(R_i) = 0$ and $\|R_i\|_X \leq M$ almost surely, then
 \begin{equation}\label{iid-random-bound-type-2}
    \mathbb{E}\left\|\sum_{i=1}^nR_i\right\|_X^2 \leq 4nC_{2,X}^2M^2.
 \end{equation}
 Applying this to the sequence $R_i = n^{-1}(Y_n - f_\delta)$ completes the proof since $\|Y_n - f_\delta\|_X \leq \|Y_n\| + \|f_\delta\|_X \leq 2K_\mathbb{D}$ almost surely.
 
 Finally, we prove \eqref{iid-random-bound-type-2} using a symmetrization argument and the type-2 property of $X$. Let $R'_1,...,R'_n$ denote a new set of independent copies of $R_1,...,R_n$ and let $\mathbb{E}'$ denote the expectation over $R'_1,...,R'_n$ and $\mathbb{E}$ denote the expectation over the original $R_1,...,R_n$. Using the zero mean property of the $R_i$ and Jensen's inequality, we get
 \begin{equation}
     \mathbb{E}\left\|\sum_{i=1}^nR_i\right\|_X^2 \leq \mathbb{E}\mathbb{E}'\left\|\sum_{i=1}^nR_i-R'_i\right\|_X^2.
 \end{equation}
 For any fixed choice of signs $\epsilon_1,...,\epsilon_n$, the distribution of $\sum_{i=1}^n\epsilon_i(R_i-R'_i)$ is the same. This is due the fact that $R_i$ and $R_i'$ are i.i.d. and switching the sign $\epsilon_i$ is the same as swapping $R_i$ and $R_i'$. This means that
 \begin{equation}
     \mathbb{E}\mathbb{E}'\left\|\sum_{i=1}^nR_i-R'_i\right\|_X^2 = \mathbb{E}_\epsilon\mathbb{E}\mathbb{E}'\left\|\sum_{i=1}^n\epsilon_i(R_i-R'_i)\right\|_X^2,
 \end{equation}
 where $\mathbb{E}_\epsilon$ denotes an average over the Rademacher random variables $\epsilon_1,...,\epsilon_n$. Switching the order of the expectation and using the type-2 property of $X$, we get
 \begin{equation}
     \mathbb{E}\mathbb{E}'\mathbb{E}_\epsilon\left\|\sum_{i=1}^n\epsilon_i(R_i-R'_i)\right\|_X^2 \leq C_{2,X}^2\mathbb{E}\mathbb{E}'\sum_{i=1}^n \|R_i - R_i'\|_X^2.
 \end{equation}
 Finally, since $\|R_i\|_X\leq M$ almost surely, we get that $\|R_i - R_i'\|_X \leq 2M$ almost surely so that
 \begin{equation}
    \mathbb{E}\left\|\sum_{i=1}^nR_i\right\|_X^2 \leq 4nC_{2,X}^2M^2,
 \end{equation}
 which completes the proof.
\end{proof}

\section{Smoothly parameterized dictionaries}\label{main-result-1-section}
Let $X$ be a Banach space and consider a dictionary $\mathbb{D}\subset X$ which is parameterized by a smooth manifold $\mathcal{M}$, i.e. for which there exists a surjection
\begin{equation}
 \mathcal{P}:\mathcal{M}\rightarrow \mathbb{D}.
\end{equation}
In this section, we consider dictionaries $\mathbb{D}$ which are parameterized by a smooth compact manifold $\mathcal{M}$ and study how the approximation properties of the convex hull $B_1(\mathbb{D})$ depend upon the smoothness of the parameterization map $\mathcal{P}$. Specifically, we give upper bounds on the metric entropy and $n$-widths of $B_1(\mathbb{D})$, and upper bounds on the approximation rates for $B_1(\mathbb{D})$ from sparse convex combinations $\Sigma_{n,M}(\mathbb{D})$, as a function of the degree of smoothness of the parameterization map $\mathcal{P}$.

We being by introducing the relevant notion of smoothness. Throughout this section, $\mathcal{M}$ will denote a smooth manifold with a smooth boundary. For a given $s > 0$ and domain $\Omega\subset \mathbb{R}^d$, we consider the Lipschitz space $\lip(s,L^\infty(\Omega))$ (see \cite{lorentz1996constructive} Chapter 2 for instance) with semi-norm given by
\begin{equation}
 |f|_{\lip(s,L^\infty(\Omega))} = \sup_{x,y\in \Omega} \frac{|D^kf(x) - D^kf(y)|}{|x-y|^\alpha}.
\end{equation}
Here $s = k+\alpha$, $k$ is an integer, $D^k$ represents the $k$-th derivative (tensor), and $0 < \alpha\leq 1$. If $s$ is not an integer, it is well-known that the space $\lip(s,L^\infty(\Omega))$ is equivalent to the Besov space $B_{\infty,\infty}^s(\Omega)$, but when $s$ is an integer, the Lipschitz space is slightly smaller \cite{lorentz1996constructive}. The first step is to extend this definition to Banach space valued functions $f$.
\begin{definition}\label{banach-space-map-smoothness-definition}
 Let $X$ be a Banach space, $U\subset \mathbb{R}^d$ an open set and $s > 0$. A function $\mathcal{F}:U\rightarrow X$ is of smoothness class $s$, which we write $\mathcal{F}\in \lip_\infty(s, X)$, if for every $\xi\in X^*$, the function $f_\xi:U\rightarrow \mathbb{R}$ defined by
 \begin{equation}
  f_\xi(x) = \langle\xi, \mathcal{F}(x)\rangle
 \end{equation}
 satisfies
 \begin{equation}
  |f_\xi|_{\lip(s,L^\infty(\Omega))} \leq K\|\xi\|_{X^*}
 \end{equation}
 for a constant $K < \infty$. The smallest constant $K$ above is the semi-norm $|\mathcal{F}|_{\lip_\infty(s, X)}$.

\end{definition}
In order to apply our method to more general dictionaries, we generalize this definition to allow the domain to be a smooth manifold.
\begin{definition}
 Let $X$ be a Banach space, $\mathcal{M}$ a smooth $d$-dimensional manifold, and $s > 0$. A map $\mathcal{P}:\mathcal{M}\rightarrow X$ is of smoothness class $s$, which we write $\mathcal{P}\in \lip^\mathcal{M}_\infty(s, X)$, if for each coordinate chart $(U,\phi)$ we have $\mathcal{P}\circ \phi\in \lip_\infty(s, X)$. 
 \end{definition}

 To illustrate this definition, we consider the dictionary $\mathbb{P}^d_k$ with respect to $X = L^p(\Omega)$.
 \begin{lemma}\label{p-k-d-smoothness-lemma}
  Let $\mathcal{M} = S^{d-1}\times [c_1,c_2]$, $1\leq p < \infty$, and consider the parameterization map $\mathcal{P}_k^d:\mathcal{M}\rightarrow L^p(\Omega)$ given by
  \begin{equation}
   \mathcal{P}^d_k(\omega,b) = \sigma_k(\omega\cdot x + b)\in L^p(\Omega).
  \end{equation}
  Then $\mathcal{P}_k^d\in \lip^\mathcal{M}_\infty\left(k+\frac{1}{p}, L^p(\Omega)\right)$.
 \end{lemma}
 \begin{proof}
  Let $\xi(x)\in L^q(\Omega)$ with $\frac{1}{p} + \frac{1}{q} = 1$. Then we have
  \begin{equation}
   f_\xi(\omega,b) = \int_{\Omega} \sigma_k(\omega\cdot x + b)\xi(x)dx.
  \end{equation}
  We view $\mathcal{M}$ as embedded into $\mathbb{R}^{d+1}$ and calculate the derivatives
  \begin{equation}
  \begin{split}
   \frac{\partial^k}{\partial \omega_{i_1}\cdots\partial \omega_{i_j}\partial b^{k-j}}f_\xi(\omega,b) &= \int_{\Omega} \frac{\partial^k}{\partial \omega_{i_1}\cdots\partial \omega_{i_j}\partial b^{k-j}}\sigma_k(\omega\cdot x + b)\xi(x)dx\\
   &= k!\int_{\Omega} \left(\prod_{l=1}^j x_{i_l}\right)\sigma_0(\omega\cdot x + b)\xi(x)dx,
   \end{split}
  \end{equation}
  since $\sigma_k^{(k)} = k!\sigma_0$. Because $\Omega$ is a bounded domain, $\left(\prod_{l=1}^j x_{i_l}\right)\in L^\infty(\Omega)$ and so there exists a constant $C(k,\Omega)$ such that for all indicies $i_1,..,i_j$ we have
  \begin{equation}
   \left\|k!\left(\prod_{l=1}^j x_{i_l}\right)\xi(x)\right\|_{L^q(\Omega,dx)} \leq C(k,\Omega)\|\xi\|_{L^q(\Omega)}.
  \end{equation}
  Then we have
  \begin{equation}
  \begin{split}
   &\left|\frac{\partial^k}{\partial \omega_{i_1}\cdots\partial \omega_{i_j}\partial b^{k-j}}f_\xi(\omega,b) - \frac{\partial^k}{\partial \omega_{i_1}\cdots\partial \omega_{i_j}\partial b^{k-j}}f_\xi(\omega',b')\right| \leq \\&k!\int_{\Omega} \left(\prod_{l=1}^j x_{i_l}\right)|\sigma_0(\omega\cdot x + b) - \sigma_0(\omega'\cdot x + b')|\xi(x)dx \leq \\
   &C(k,\Omega)\|\xi\|_{L^q(\Omega)}\|\sigma_0(\omega\cdot x + b) - \sigma_0(\omega'\cdot x + b')\|_{L^p(\Omega)},
   \end{split}
  \end{equation}
  which means that
  \begin{equation}
   \left|D^kf_\xi(\omega,b) - D^kf_\xi(\omega',b')\right| \lesssim_{k,\Omega} \|\xi\|_{L^q(\Omega)}\|\sigma_0(\omega\cdot x + b) - \sigma_0(\omega'\cdot x + b')\|_{L^p(\Omega)}.
  \end{equation}
  The proof will be complete if we can show that
  \begin{equation}
   \|\sigma_0(\omega\cdot x + b) - \sigma_0(\omega'\cdot x + b')\|_{L^p(\Omega)} \lesssim_{p,\Omega} (|\omega - \omega'| + |b - b'|)^{\frac{1}{p}}.
  \end{equation}
  This follows since the function $\sigma_0(\omega\cdot x + b) - \sigma_0(\omega'\cdot x + b')$ is zero except on the set
  \begin{equation}
   S = \{x\in \Omega:\omega\cdot x + b < 0\leq \omega'\cdot x + b'\}\cup \{x\in\Omega:\omega'\cdot x + b' < 0\leq \omega\cdot x + b\},
  \end{equation}
  where it is $\pm 1$. The set $S$ is a wedge or strip within $\Omega$ of width proportional to $|\omega - \omega'| + |b - b'|$ (see also \cite{makovoz1996random}), and thus we have
  \begin{equation}
   \|\sigma_0(\omega\cdot x + b) - \sigma_0(\omega'\cdot x + b')\|_{L^p(\Omega)}^p = |S| \lesssim_{\Omega} |\omega - \omega'| + |b - b'|,
  \end{equation}
  which completes the proof.

 \end{proof}

In the following analysis, it will be convenient to use the following technical lemma which allows us to reduce to the case where $\mathcal{M}$ is a $d$-dimensional cube.
\begin{lemma}\label{image-of-union-of-cubes-lemma}
 Suppose that $\mathcal{M}$ is a d-dimensional compact smooth manifold (potentially with boundary) and we are given a parameterization $\mathcal{P}\in \lip_\infty(s, X)$. Then there exist finitely many maps $\mathcal{P}_j:[-1,1]^d\rightarrow X$, $j=1,...,T$, such that $\mathcal{P}_j\in \lip_\infty(s, X)$ for each $j$ and
 \begin{equation}
  \mathcal{P}(\mathcal{M}) = \bigcup_{j=1}^T \mathcal{P}_j([-1,1]^d).
 \end{equation}

\end{lemma}
\begin{proof}
 Let $x\in \mathcal{M}$ be an interior point. Take a chart $\phi_x:U_x\rightarrow\mathcal{M}$ containing $x$. Here $U_x$ may be homeomorphic to the upper half-plane with boundary if the chart contains boundary points. Let $T_x:C\rightarrow U_x$ be a smooth injective map from the cube $C = [-1,1]^d$ to $U_x$ such that $\phi^{-1}_x(x)$ is in the interior of $T_x(C)$. Such a map can clearly always be found since $\phi_x^{-1}(x)\in \text{int}(U_x)$ as $x\in \text{int}(\mathcal{M})$.  Define $V_x = \phi_x\circ T_x(C^o)$, where $C^o$ is the interior of $C$.
 
 For boundary points $x\in \mathcal{M}$ we take a chart $U_x\rightarrow\mathcal{M}$ containing $x$. In this case $U_x$ is homeomorphic to the upper half-plane with boundary and $\phi_x^{-1}(x)$ is a boundary point of $U_x$. Since the boundary is smooth, we may in this case find a smooth injective map $T_x:C\rightarrow U_x$ such that $\phi^{-1}_x(x)\in \{-1\}\times (-1,1)^{d-1}\subset C$. In this case we define $V_x = \phi_x\circ T_x(C^o\cup \{-1\}\times (-1,1)^{d-1})$.
 
 In either case, we have $x\in V_x$ and that $V_x$ is relatively open in $\mathcal{M}$. Since $\mathcal{M}$ is compact, it can be covered by $V_{x_j}$ for finitely many $x_1,...,x_T$. The maps $\mathcal{P}_j = \mathcal{P}\circ\phi_{x_j}\circ T_{x_j}$ satisfy the conclusion of the Lemma. Indeed, since $T_{x_i}$ is smooth and by definition $\mathcal{P}\circ\phi_i\in \lip_\infty(s,X)$, it is easy to see that $\mathcal{P}_j\in \lip_\infty(s,X)$. In addition, by construction the images $\mathcal{P}_j(C)$ cover the image $\mathcal{P}(\mathcal{M})$.
\end{proof}

\subsection{Polynomial Interpolation Error Bounds}\label{polynomial-interpolation-section}
The significance of the smoothness definition \ref{banach-space-map-smoothness-definition} lies in its relationship with approximation by polynomial interpolation. Let use briefly review some basic facts concerning polynomial interpolation. 

To set the stage, let $U\subset \mathbb{R}^d$ be an open domain, let $\Pi_k^d$ denote the space of polynomials of degree at most $k$ in $d$ variables, set $M = \dim{\Pi_k^d} = \binom{k+d}{d}$, and let $x_1,...,x_M\subset U$ be a set of points which is unisolvent for the space $\Pi_k^d$, i.e. such that no polynomial in $\Pi_k^d$ vanishes at all of the $x_i$. In one dimension, it is well-known that any $k$ distinct points are unisolvent for the space $\Pi_k^1$. It is also possible to explicitly give sets of $\binom{d+k}{k}$ points which are unisolvent for the space $\Pi_k^d$ for $d > 1$ \cite{chung1977lattices,nicolaides1972class,ciarlet1972general}. 

In this setting, we can find Lagrange polynomials $l_1,...,l_M\in \Pi_k^d$ such that $l_i(x_j) = \delta_{ij}$. 
Given values $y_1,...,y_M\in \mathbb{R}$ at the points $x_1,...,x_M$ the unique polynomial in $\Pi_k^d$ interpolating these values is given by
\begin{equation}
    \sum_{i=1}^My_il_i\in \Pi_k^d.
\end{equation}
The norm of the interpolation map as a map from $\ell^\infty_M$ to $C(U)$, called the Lesbesgue constant, is given by
\begin{equation}
    \Lambda_k^d(U,\{x_1,...,x_M\}) = \sup_{x\in U} \sum_{i=1}^M |l_i(x)|.
\end{equation}
An elementary, yet critical fact about the Lesbesgue constant is its invariance under invertible affine transformations, which is immediate since the space of polynomials $\Pi_k^d$ is invariant under such transformations.
\begin{lemma}\label{scaling-argument-lemma}
 Let $A$ be an invertible linear map on $\mathbb{R}^d$ and $b\in \mathbb{R}^d$ a fixed vector. Then we have
 \begin{equation}
     \Lambda_k^d(U,\{x_1,...,x_M\}) = \Lambda_k^d(AU+b,\{Ax_1+b,...,Ax_M+b\}).
 \end{equation}
 Here $AU + b = \{Ax + b,~x\in U\}$.
\end{lemma}

Next, we recall the classical Bramble-Hilbert lemma \cite{bramble1970estimation} (see also \cite{xu1982error,xu2013estimate}), which bounds the polynomial interpolation error for functions $f\in \lip(s,L^\infty(U))$.
\begin{lemma}\label{bramble-hilbert-lemma}
 Let $s = k + \alpha$ with $k\in \mathbb{Z}$ and $\alpha\in (0,1]$ and suppose that $f\in \lip(s,L^\infty(U))$ for a convex domain $U$. Let
 \begin{equation}
     f_I(x) = \sum_{i=1}^M f(x_i)l_i(x)\in \Pi_k^d
 \end{equation}
 denote the polynomial which interpolates the values of $f$ at the points $x_1,...,x_M$. Then for any $y\in U$ we have
 \begin{equation}
     |f(y) - f_I(y)| \leq C|f|_{\lip(s,L^\infty(\Omega))}[\diam(U)]^s\Lambda_k^d(U,\{x_1,...,x_M\}),
 \end{equation}
 where the constant $C$ only depends upon $k$ and $d$.
\end{lemma}
\begin{proof}
 For each $i=1,...,M$ consider the function $r_i(t) = f(y + t(x_i - y))$. Taylor expanding $r_i$ about $t=0$, we obtain 
 \begin{equation}\label{r-taylor-expansion}
     r_i(1) - r_i(0) = \sum_{j=1}^{k} \frac{1}{j!}r_i^{(j)}(0) + \frac{1}{(k-1)!}\int_0^1 [r_i^{(k)}(t) - r_i^{(k)}(0)]t^{k-1}dt.
 \end{equation}
 Next, note that by construction $r_i^{(j)}(0) = D^jf(y)\cdot (x_i - y)^{\otimes j}$ and $r_i^{(k)}(t) = D^kf(y + t(x_i - y))\cdot (x_i - y)^{\otimes k}$. Plugging this into equation \eqref{r-taylor-expansion}, we get
 \begin{equation}
     f(x_i) - f(y) = \sum_{j=1}^{k} \frac{1}{j!}D^jf(y)\cdot (x_i - y)^{\otimes j} + \left(\frac{1}{(k-1)!}\int_0^1 [D^kf(y + t(x_i - y)) - D^kf(y)]t^{k-1}dt\right)\cdot (x_i - y)^{\otimes k}.
 \end{equation}
 We now multiply this equation by $l_i(y)$ and sum over $y$ to obtain
 \begin{equation}\label{interpolation-bound-eq-728}
     f_I(y) - f(y) = \sum_{i=1}^Ml_i(y)\left(\int_0^1 [D^kf(y + t(x_i - y)) - D^kf(y)]t^{k-1}dt\right)\cdot (x_i - y)^{\otimes k}.
 \end{equation}
 Here we have used the identities
 \begin{equation}
     \sum_{i=1}^Ml_i(y) = 1,~\sum_{i=1}^Ml_i(y)(x_i - y)^{\otimes j} = 0,
 \end{equation}
 for $j=1,...,k$.
 The first of these identities holds since left hand side is the interpolation of the constant function $1$ evaluated at $y$. The second holds since the left hand side is the interpolation of the degree $j$ polynomial $p(x) = (x - y)^{\otimes j}$ evaluated at $y$. Since $j \leq k$, this polynomial is reproduced exactly and so we get $p(y) = 0$.
 
 Since $f\in \lip(s,L^\infty(U))$ and $U$ is convex so that $y + t(x_i - y)\in U$ for $t\in [0,1]$, we get
 \begin{equation}
     |D^kf(y + t(x_i - y)) - D^kf(y)| \leq |f|_{\lip(s,L^\infty(\Omega))}(t|x_i - y|)^\alpha.
 \end{equation}
 Since also $|(x_i - y)^{\otimes k}| \leq C(k,d)|x_i - y|^k$, we get
 \begin{equation}
 \begin{split}
     \left(\int_0^1 [D^kf(y + t(x_i - y)) - D^kf(y)]t^{k-1}dt\right)\cdot (x_i - y)^{\otimes k} &\leq C(k,d)|f|_{\lip(s,L^\infty(\Omega))}|x_i - y|^{k+\alpha}\\ &\leq C(k,d)|f|_{\lip(s,L^\infty(\Omega))}[\diam(U)]^s.
\end{split}
 \end{equation}
 for each $i=1,...M$. Plugging this into \eqref{interpolation-bound-eq-728}, finally get
 \begin{equation}
 \begin{split}
     |f_I(y) - f(y)| &\leq C(k,d)|f|_{\lip(s,L^\infty(\Omega))}[\diam(U)]^s\sum_{i=1}^M|l_i(y)|\\ &\leq C(k,d)|f|_{\lip(s,L^\infty(\Omega))}[\diam(U)]^s\Lambda_k^d(U,\{x_1,...,x_M\}),
\end{split}
 \end{equation}
 as desired.
\end{proof}

Given Banach space values $y_1,...,y_M\in X$ at the points $x_1,...,x_M$, we can analogously form the interpolating polynomial $P_k:U\rightarrow X$ by
\begin{equation}
    P_k(x) = \sum_{i=1}^My_il_i(x).
\end{equation}
The next lemma shows that if the $y_i = \mathcal{F}(x_i)$ are the values of a map $\mathcal{F}\in \lip_\infty(s, X)$, then the Bramble-Hilbert lemma holds in the Banach space setting.
\begin{lemma}\label{polynomial-interpolation-lemma}
 Let $U\subset \mathbb{R}^d$ be a convex domain and suppose that a map $\mathcal{F}:U\rightarrow X$ satisfies $\mathcal{F}\in \lip_\infty(s, X)$ where $s = k+\alpha$ with $k$ an integer and $\alpha\in (0,1]$. Let
 \begin{equation}
     \mathcal{F}_I(x) = \sum_{i=1}^M \mathcal{F}(x_i)l_i(x)
 \end{equation}
 denote the polynomial which interpolates the values of $\mathcal{F}$ at the points $x_1,...,x_M$. Then for any $y\in U$ we have
 \begin{equation}
     \|\mathcal{F}(y) - \mathcal{F}_I\|_X \leq C|\mathcal{F}|_{\lip(s,X)}[\diam(U)]^s\Lambda_k^d(U,\{x_1,...,x_M\}),
 \end{equation}
 where the constant $C$ only depends upon $k$ and $d$.
\end{lemma}
\begin{proof}
 The proof is by duality. Let $\xi\in X^*$ and note that by definition the function $f_\xi = \langle \xi, \mathcal{F}(x)\rangle$ satisfies
 \begin{equation}
     |f_\xi|_{\lip(s,L^\infty(\Omega))} \leq |\mathcal{F}|_{\lip(s,X)}\|\xi\|_{X^*}.
 \end{equation}
 We also have that
 \begin{equation}
     f_{\xi,I}:=\langle\xi, \mathcal{F}_I\rangle = \sum_{i=1}^M\langle\xi, \mathcal{F}(x_i)\rangle l_i(x) = \sum_{i=1}^M f_\xi(x_i)l_i(x)
 \end{equation}
 is the interpolation of $f_\xi$ at the points $x_1,...,x_M$. Applying the Bramble-Hilbert lemma \ref{bramble-hilbert-lemma} to the function $f_\xi$, we get
 \begin{equation}
     |\langle\xi, \mathcal{F}_I(y) - \mathcal{F}(y)\rangle| = |f_{\xi,I}(y) - f_{\xi}(y)| \leq C|\mathcal{F}|_{\lip(s,X)}\|\xi\|_{X^*}[\diam(U)]^s\Lambda_k^d(U,\{x_1,...,x_M\})
 \end{equation}
 where the constant $C$ only depends upon $k$ and $d$. Since this is true for all $\xi\in X^*$, the result follows.
\end{proof}

\subsection{Approximation rates for smoothly parameterized dictionaries}
Next, we give upper bounds on the approximation rates of $B_1(\mathbb{D})$ from sparse convex combinations $\Sigma_{n,M}(\mathbb{D})$ for smoothly parameterized dictionaries. We have the following theorem. 
\begin{theorem}\label{upper-bound-theorem}
 Let $s > 0$ and $X$ be a type-$2$ Banach space. Suppose that $\mathcal{M}$ is a compact $d$-dimensional smooth manifold, $\mathcal{P}\in \lip_\infty^\mathcal{M}(s, X)$, and the dictionary $\mathbb{D}\subset \mathcal{P}(\mathcal{M})$. Then there exists an $M > 0$ such that for $f\in B_1(\mathbb{D})$ we have
 \begin{equation}
 \inf_{f_n\in \Sigma_{n,M}(\mathbb{D})} \|f - f_n\|_{X} \lesssim n^{-\frac{1}{2} - \frac{s}{d}}.
 \end{equation}
 Here both $M$ and the implied constant depend only upon $s$, $d$, the parameterization map $\mathcal{P}$, and the type-$2$ constant of the space $X$.
\end{theorem}
We note that although the implied constants here are independent of $n$, they may be quite large and accurately estimating them would require careful consideration of the structure of the manifold $\mathcal{M}$ and the parameterization map $\mathcal{P}$. The proof of this theorem is a higher-order generalization of the stratified sampling argument \cite{makovoz1996random,klusowski2018approximation}, which corresponds to the $k=0$ case.
Before proving this theorem, we note a corollary obtained when applying it to the dictionary $\mathbb{P}^d_k$. 
\begin{theorem}\label{relu-k-rate-corollary}
 Let $k\geq 0$ and $2\leq p < \infty$. Then there exists an $M = M(p,k,d) > 0$ such that for all $f\in B_1(\mathbb{P}^d_k)$ we have
 \begin{equation}
  \inf_{f_n\in \Sigma_{n,M}(\mathbb{P}^d_k)} \|f - f_n\|_{L^p(\Omega)} \lesssim n^{-\frac{1}{2}-\frac{pk+1}{pd}}.
 \end{equation}
\end{theorem}
In the case $p=2$, we obtain in particular that $\alpha(k,d) \geq \frac{2k+1}{2d}$ in \eqref{optimal-approx-equation}. We will show in Section \ref{lower-bounds-section} that this rate is sharp when $p=2$. However, we expect that this rate can be improved when $2 < p < \infty$ using the techniques of geometric discrepancy theory \cite{matouvsek1996improved,bach2017breaking,matousek1999geometric}.
\begin{proof}
 This follows immediately from Theorem \ref{upper-bound-theorem} given the smoothness condition of the map $\mathcal{P}^d_k$ proven in Lemma \ref{p-k-d-smoothness-lemma} and the fact that $S^{d-1}\times [c_1,c_2]$ is a compact $d$-dimensional manifold. 
\end{proof}

\begin{proof}[Proof of Theorem \ref{upper-bound-theorem}]
 We apply lemma \ref{image-of-union-of-cubes-lemma} to $\mathcal{P}$ and $\mathcal{M}$ to obtain a collection of maps $\mathcal{P}_j:C:=[-1,1]^d\rightarrow X$ such that $\mathbb{D} = \cup_{j=1}^T\mathcal{P}_j(C)$ and $\mathcal{P}_j\in \lip_\infty(s, X)$. We remark that using cubes here is not strictly necessary, but we do this for convenience since cubes can be easily subdivided in a straightforward manner.
 
 It suffices to prove the result for $\mathbb{D} = \mathbb{D}_j := \mathcal{P}_j(C)$. This follows since $B_1(\mathbb{D}) = \text{conv}(\cup_{j=1}^T B_1(\mathbb{D}_j))$ and given a convex combination $f = \alpha_1f_1 +\cdots + \alpha_Tf_T$ with $f_j\in B_1(\mathbb{D}_j)$ and $\sum_{i=1}^T \alpha_j= 1$, we get
 \begin{equation}
  \inf_{f_n\in \Sigma_{Tn,M}(\mathbb{D})} \|f - f_n\|_{H} \leq \sum_{j=1}^T \alpha_j\inf_{f_{n,j}\in \Sigma_{n,M}(\mathbb{D}_j)} \|f_j - f_{n,j}\|_{X},
 \end{equation}
 which easily follows by setting $f_n = \sum_{j=1}^T\alpha_jf_{n,j}$ and noting that $\mathbb{D} = \cup_{j=1}^T \mathbb{D}_j$.
 
So in what follows we consider $\mathbb{D} = \mathbb{D}_j$, $\mathcal{P} = \mathcal{P}_j$ and $\mathcal{M} = C$. In other words, we assume without loss of generality that $T=1$ (at the cost of introducing a constant which depends upon $T$ and thus upon $\mathcal{P}$ and $\mathcal{M}$).
 
 Now let $f\in B_1(\mathbb{D})$ and $\delta > 0$. Then there exists a convex combination (with potentially very large $N:=N_\delta$)
 \begin{equation}\label{eq-176}
  f_\delta = \sum_{i=1}^N a_id_i,
 \end{equation}
 with $d_i\in \mathbb{D}$, $\sum|a_i| \leq 1$, and $\|f - f_\delta\|_X  < \delta$. Since $\mathbb{D} = \mathcal{P}(C)$, each $d_i = \mathcal{P}(z_i)$ for some $z_i\in C$, so we get 
 \begin{equation}
  f_\delta = \sum_{i=1}^{N} a_i\mathcal{P}(z_i).
 \end{equation}
 We remark that in what follows all implied constants will be independent of $n$ and $\delta$.
 
 Let $n \geq 1$ be given and subdivide the cube $C$ into $n$ sub-cubes $C_1,...,C_n$ such that each $C_r$ has diameter $O(n^{-\frac{1}{d}})$. This can easily be done by considering a uniform subdivision in each direction. (This is also why we chose to use cubes in this construction.)
 
 We proceed to approximate the map $\mathcal{P}$ by a piecewise polynomial on the subcubes $C_1,...,C_n$. To this end, let $M=\binom{d+k}{k}$ and $x_1,...,x_{M}\in C$ a set of points which is unisolvent for the space of polynomials of degree at most $k$ in $d$ variables and let $l_1,...,l_M$ be the associated Lagrange interpolation polynomials, as discussed in Section \ref{polynomial-interpolation-section}. Here the integer $k$ is determined by $s = k+\alpha$ with $\alpha\in (0,1]$.
 
 For each cube $C_r$, denote by $x_1^r,...,x_M^r$ and $l_1^r,...,l_M^r$ the image of the interpolation points $x_1,...,x_M$ on the cube $C_r$ and their associated Lagrange polynomials. We rewrite $f_\delta$ as
 \begin{equation}\label{eq-194}
  f_\delta = \sum_{r=1}^{n}\sum_{z_i\in C_r} \mathcal{P}(z_i) = \sum_{r=1}^{n}\sum_{z_i\in C_r}a_i\mathcal{P}_{r,I}(z_i) +  \sum_{l=1}^{n}\sum_{z_i\in C_r}a_i\mathcal{E}_r(z_i),
 \end{equation}
 where the polynomial interpolation in the cube $C_r$ is given by
 \begin{equation}\label{eq-198}
  \mathcal{P}_{r,I}(z) = \sum_{i=1}^{M} \mathcal{P}(x_i^r)l_i^r(z),
 \end{equation}
 and the error in the approximation is given by
 \begin{equation}
  \mathcal{E}_r(z) = \mathcal{P}(z) - \mathcal{P}_{r,I}(z).
 \end{equation}
 
 We use the Banach space Bamble-Hilbert Lemma \ref{polynomial-interpolation-lemma} and Maurey's sampling argument (Theorem \ref{maurey-sampling-argument-thm}) to bound the second term in \eqref{eq-194}.  We apply Theorem \ref{maurey-sampling-argument-thm} with the dictionary $\mathbb{D}_{\mathcal{E}} = \{\mathcal{E}_r(z_i),~z_i\in C_r\}$ to the term
 \begin{equation}
     \sum_{l=1}^{n}\sum_{z_i\in C_r}a_i\mathcal{E}_r(z_i)\in B_1(\mathbb{D}_{\mathcal{E}}).
 \end{equation}
 This yields the existence of an $n$-term convex combination
 \begin{equation}
     f'_n = \frac{1}{n}\sum_{s=1}^n \mathcal{E}_{r_s}(z_{i_s})
 \end{equation}
 satisfying
 \begin{equation}
     \left\|f'_n - \sum_{l=1}^{n}\sum_{z_i\in C_r}a_i\mathcal{E}_r(z_i)\right\|_X \leq CK_{\mathbb{D}_\mathcal{E}}n^{-\frac{1}{2}},
 \end{equation}
 where $C$ only depends upon the space $X$.
 Lemma \ref{polynomial-interpolation-lemma} implies that
 \begin{equation}
     K_{\mathbb{D}_\mathcal{E}} \leq \sup_{z\in C_r}\|\mathcal{E}_r(z)\|_X \leq C(k,d)|\mathcal{P}|_{\lip(s,X)}[\diam(C_r)]^s\Lambda_k^d(C_r,x^r_1,...,x^r_M) = Cn^{-\frac{s}{d}},
 \end{equation}
 where $C$ is independent of $n$. This holds since by Lemma \ref{scaling-argument-lemma} the Lebesgue constant satisfies $$\Lambda_k^d(C_r,x^r_1,...,x^r_M) = \Lambda_k^d(C,x_1,...,x_M)$$ and is thus independent of $n$. Setting
 \begin{equation}
     f_n = \sum_{r=1}^{n}\sum_{z_i\in C_r}a_i\mathcal{P}_{r,I}(z_i) + f'_n,
 \end{equation}
 we thus obtain
 \begin{equation}
     \|f_n - f_\delta\|_X = \left\|f'_n - \sum_{l=1}^{n}\sum_{z_i\in C_r}a_i\mathcal{E}_r(z_i)\right\|_X \lesssim n^{-\frac{1}{2}-\frac{s}{d}},
 \end{equation}
 where the implied constant is independent of $n$. Finally, we observe that
 \begin{equation}\label{f-n-width-norm-bound}
     f_n = \sum_{r=1}^{n}\sum_{z_i\in C_r}a_i\mathcal{P}_{r,I}(z_i) + \frac{1}{n}\sum_{s=1}^n (\mathcal{P}(z_{i_s}) - \mathcal{P}_{r_s,I}(z_{i_s})) \in \Sigma_{(M+1)n,2K+1}(\mathbb{D}),
 \end{equation}
 for $K := \sup_{x\in C_r}\sum_{i=1}^{M} |l^r_i(x)| = \Lambda_k^d(C_r,x^r_1,...,x^r_M) = \Lambda_k^d(C,x_1,...,x_M)$. This holds since the interpolating polynomial $\mathcal{P}_{r,I}(z_i)$ only involves evaluations of the map $\mathcal{P}$ at the fixed interpolation points $x_1^r,...,x_M^r$ in the cube $C_r$ and the coefficients of those evaluations are bounded in $\ell^1$ by the Lebesgue constant $\Lambda_k^d(C,x_1,...,x_M)$. Since there are $n$ cubes $C_1,...,C_n$ which each contain $M$ interpolation points, there are $Mn$ interpolation points in total, so we have
 \begin{equation}
  \sum_{r=1}^{n}\sum_{z_i\in C_r}a_i\mathcal{P}_{r,I}(z_i) - \frac{1}{n}\sum_{s=1}^n \mathcal{P}_{r_s,I}(z_{i_s})  \in \Sigma_{Mn,2K}(\mathbb{D}),
 \end{equation}
 from which \eqref{f-n-width-norm-bound} follows. Since $\delta > 0$ was arbitrary, this completes the proof.

\end{proof}

\subsection{Metric entropy bounds for smoothly parameterized dictionaries}\label{entropy-section}
Next, we bound the metric entropy of $B_1(\mathbb{D})$ for smoothly parameterized dictionaries $\mathbb{D}$. We first observe that the approximation rate proven in Theorem \ref{upper-bound-theorem} implies a bound on the metric entropy via Theorem \ref{dictionary-carls-theorem} (see also the proofs of Theorems 4 in \cite{makovoz1996random} and in \cite{klusowski2018approximation}). Specifically, under the assumptions of Theorem \ref{upper-bound-theorem} we get
\begin{equation}\label{carls-approximation-rate-entropy}
    \epsilon_{n\log{n}}(B_1(\mathbb{D})) \lesssim n^{-\frac{1}{2}-\frac{s}{d}}.
\end{equation}
The main result in this section is that the logarithmic factor in \eqref{carls-approximation-rate-entropy} can be removed.

\begin{theorem}\label{entropy-upper-bound-theorem}
 Let $s > 0$ and $X$ be a type-$2$ Banach space. Suppose that $\mathcal{M}$ is a compact $d$-dimensional smooth manifold, $\mathcal{P}\in \lip_\infty^\mathcal{M}(s, X)$, and the dictionary $\mathbb{D}\subset \mathcal{P}(\mathcal{M})$. Then
 \begin{equation}\label{entropy-bound-equation}
 \epsilon_n(B_1(\mathbb{D}))_X \lesssim n^{-\frac{1}{2} - \frac{s}{d}}.
 \end{equation}
 Here the implied constant depends only upon $s$, $d$, the parameterization map $\mathcal{P}$, and the type-$2$ constant of the space $X$.
\end{theorem}

We note that combined with the bounds in Lemma \ref{p-k-d-smoothness-lemma}, Theorem \ref{entropy-upper-bound-theorem} implies the following bound for the variation space corresponding to shallow ReLU$^k$ networks
\begin{equation}
 \epsilon_n(B_1(\mathbb{P}_k^d))_{L^2(\Omega)} \lesssim n^{-\frac{1}{2} - \frac{2k+1}{2d}}.
\end{equation}
In Section \ref{lower-bounds-section} we will show that this rate is sharp up to a constant factor.

To prove Theorem \ref{entropy-upper-bound-theorem} we will use the following two lemmas. The first is a triangle inequality for the entropy numbers.
\begin{lemma}[see \cite{lorentz1996constructive} Section 15.7]\label{triangle-inequality-entropy-lemma}
 Let $A,B\subset X$. Then for any $0 \leq m\leq n$
 \begin{equation}
  \epsilon_n(A+B) \leq \epsilon_m(A) + \epsilon_{n-m}(B).
 \end{equation}

\end{lemma}
\begin{proof}
 If balls of radius $\epsilon_m(S)$ around $s_1,...,s_{2^m}\in X$ cover $A$ and balls of radius $\epsilon_{n-m}(T)$ around $t_1,...,t_{2^{n-m}}\in X$ cover $B$, then balls of radius $\epsilon_m(S) + \epsilon_{n-m}(T)$ around the $2^n$ points $s_i + t_j$ cover $A + B$. If the infemum in the entropy is not achieved, then a simple limiting argument can be used to complete the proof.
\end{proof}

The second, due to Carl (Proposition 1 in \cite{carl1985inequalities}), gives a bound on the metric entropy of the convex hull of a finite dictionary $\mathbb{D}\subset X$.
\begin{lemma}\label{carls-lemma}
 Let $X$ be a type-$2$ Banach space and $\mathbb{D}\subset X$ a dictionary with $n$ elements, i.e. $\mathbb{D} = \{d_1,...,d_n\}$. Set $K_\mathbb{D} = \max_{i=1,...,n}\|d_i\|_X$. Then
 \begin{equation}
  \epsilon_m(B_1(\mathbb{D})) \lesssim \begin{cases} 
         K_\mathbb{D} & m=0 \\
          \sqrt{1+\log{\frac{n}{m}}}m^{-\frac{1}{2}}K_\mathbb{D} & 1\leq m\leq n \\
          2^{-\frac{m}{n}}n^{-\frac{1}{2}}K_\mathbb{D} & m > n,
       \end{cases}
 \end{equation}
 where the implied constant only depends upon the type-$2$ constant of $X$.
\end{lemma}
 \begin{proof}[Proof of Theorem \ref{entropy-upper-bound-theorem}]
 As in the proof of Theorem \ref{upper-bound-theorem}, we begin by reducing to the case where $\mathcal{M} = C := [-1,1]^d$ is the cube. Using Lemma \ref{image-of-union-of-cubes-lemma}, we see that there exists an integer $T$ and a collection of maps $\mathcal{P}_j:C\rightarrow X$ such that $\mathbb{D}\subset \cup_{j=1}^T\mathcal{P}_j(C)$ and $\mathcal{P}_j\in \lip_\infty(s,X)$. Again, using the cube here is not strictly necessary, but simplifies the argument somewhat.
 
 Now $B_1(\mathbb{D})\subset \sum_{j=1}^T B_1(\mathcal{P}_j(C))$ and applying Lemma \ref{triangle-inequality-entropy-lemma} implies that
 \begin{equation}
     \epsilon_{Tn}(B_1(\mathbb{D})) \leq \sum_{j=1}^T\epsilon_n(B_1(\mathcal{P}_j(C))).
 \end{equation}
 Thus, at the cost of a constant factor it suffices to prove Theorem \ref{entropy-bound-equation} for $\mathbb{D} = \mathcal{P}_j(C)$ for each $j$. So we set $\mathcal{P} = \mathcal{P}_j$ and consider the case where $\mathcal{M} = C$ and $\mathbb{D} = \mathcal{P}(C)$.
 
 Let $s = k + \alpha$ with $k$ and integer and $\alpha\in (0,1]$. Set $M = \binom{d+k}{k}$ and choose $M$ points $x_1,...,x_M\in C$ which are unisolvent for the space of polynomials of degree at most $k$ in $d$ variables nad let $l_1,...,l_M$ be the associated Lagrange interpolation polynomials. For each integer $i\geq 0$, we subdivide the cube into $2^{di}$ sub-cubes $C_1,...,C_{2^{di}}$ of side length $2^{-i}$. We let $\mathcal{P}_i$ denote the piecewise degree $k$ interpolation of the map $\mathcal{P}$ on the sub-cubes $C_r$. Specifically, for $z\in C$, we denote by $r_i(z)\in \{1,...,2^{di}\}$ the index such that $z\in C_{r_i(z)}$ (for points on the boundary of a sub-cube where this index may not be unique we simply choose one) and define
 \begin{equation}
    \mathcal{P}_i(z) = \sum_{j=1}^M l_j^{r_i(z)}(z)\mathcal{P}(x_j^{r_i(z)}) 
 \end{equation}
 where $l_j^{r_i(z)}$ and $x_j^{r_i(z)}$ are the images of the Lagrange polynomials and interpolation points on the sub-cube $C_{r_i(z)}$ containing $z$ at discretization level $i$.
 
 Next, we define dictionaries $\mathbb{D}_i$ for $i\geq 0$ by
 \begin{equation}
     \mathbb{D}_i = \{\mathcal{P}_i(z) - \mathcal{P}_{i-1}(z),~z\in C\}.
 \end{equation}
 Here we set $\mathcal{P}_{-1}(z) = 0$. Note that
 \begin{equation}
     B_1(\mathbb{D}) \subset \overline{\sum_{i=1}^\infty B_1(\mathbb{D}_i)}.
 \end{equation}
 Indeed, by definition $B_1(\mathbb{D})$ is the closure of elements of the form
 \begin{equation}
     \sum_{l=1}^N a_l\mathcal{P}(z_l) = \sum_{i=1}^\infty \sum_{l=1}^Na_l(\mathcal{P}_i(z_l) - \mathcal{P}_{i-1}(z_l))\in \sum_{i=1}^\infty B_1(\mathbb{D}_i),
 \end{equation}
 where $\sum_{i=1}^N |a_i| \leq 1$. Using Lemma \ref{triangle-inequality-entropy-lemma} inductively this implies that for any sequence of integers $n_1,n_2,...\geq 0$ such that $\sum_{i=1}^\infty n_i = n$ we have the bound
 \begin{equation}\label{iterated-triangle-inequality-entropy-bound}
     \epsilon_n(B_1(\mathbb{D})) \leq \sum_{i=1}^\infty \epsilon_{n_i}(B_1(\mathbb{D}_i)).
 \end{equation}
 Note that since the entropy is decreasing this also holds if $\sum_{i=1}^\infty n_i \leq n$.
 
 The next step is to bound $\epsilon_{n_i}(B_1(\mathbb{D}_i))$. For this we note the following composition property of interpolation
 \begin{equation}
     \mathcal{P}_{i-1}(z) = \sum_{j=1}^M l_j^{r_i(z)}(z)\mathcal{P}_{i-1}(x_j^{r_i(z)}),
 \end{equation}
 which follows since the function $\mathcal{P}_{i-1}$ equals its interpolation on the finer grid at level $i$. Thus the dictionary $\mathbb{D}_i$ can be rewritten as
 \begin{equation}
     \mathbb{D}_i = \left\{\sum_{j=1}^M l_j^{r_i(z)}(z)[x_j^{r_i(z)} - \mathcal{P}_{i-1}(x_j^{r_i(z)})],~z\in C\right\},
 \end{equation}
 from which it follows that
 \begin{equation}
     B_1(\mathbb{D}_i)\subset \left(\max_{z\in C}\sum_{j=1}^M |l_j^{r_i(z)}(z)|\right)B_1(\overline{\mathbb{D}}_i) = \Lambda_k^d(C,\{x_1,...,x_M\})B_1(\overline{\mathbb{D}}_i),
 \end{equation}
 where $\overline{\mathbb{D}}_i$ is the finite dictionary given by
 $$
 \overline{\mathbb{D}}_i = \{x_j^r - \mathcal{P}_{i-1}(x_j^r),~j=1,...,M,~r=1,...,2^{di}\}.
 $$
 We note that the number of elements in $\overline{\mathbb{D}}_i$ is $M2^{di}$ and the Banach space Bramble-Hilbert Lemma \ref{polynomial-interpolation-lemma} implies that $K_{\overline{\mathbb{D}_i}} \lesssim 2^{-si}$.
 We now use Lemma \ref{carls-lemma} to get
 \begin{equation}\label{partial-dictionary-entropy-bound}
     \epsilon_{m}(B_1(\mathbb{D}_i)) \leq \Lambda_k^d(C,\{x_1,...,x_M\})\epsilon_m(B_1(\overline{\mathbb{D}}_i)) \lesssim \begin{cases} 
         2^{-si} & m=0 \\
          2^{-si}m^{-\frac{1}{2}}\sqrt{1 - \log{m} + di + \log M} & 1\leq m\leq M2^{di} \\
          2^{-\left(s+\frac{d}{2}\right)i}2^{-\frac{m}{M2^{di}}} & m > M2^{di},
       \end{cases}
 \end{equation}
 where the implied constant is independent of $i$ and $m$ (specifically it will only depend upon $k$, $d$, the parameterization map $\mathcal{P}$ and the type-$2$ constant of $X$).
 
 Finally, the proof is completed by substituting \eqref{partial-dictionary-entropy-bound} into \eqref{iterated-triangle-inequality-entropy-bound} and optimizing over the choice of $n_i$. This is a somewhat standard, but involved calculation (see \cite{ball1990entropy,carl1997metric,carl1999metric,matouvsek1995tight} for instance). 
 
 It suffices to prove Theorem \ref{entropy-upper-bound-theorem} for $n$ of the form $n=K2^{rd}$ for a fixed integer $K$ which will be determined later. This follows since the entropy is a decreasing function and our bound is polynomial in $n$, so extending to all values of $n$ will only increase the implied constant. For such a value of $n$ we wish to show that
 \begin{equation}
     \epsilon_n(B_1(\mathbb{D})) \lesssim 2^{-\left(s+\frac{d}{2}\right)r},
 \end{equation}
 where the implied constant is independent of $r$. Let $\delta > 0$ and choose the $n_i$ as
 \begin{equation}
     n_i = \begin{cases}
     M\left(s+\frac{d}{2}+\delta\right)(r-i+1)2^{di} & 0\leq i < r\\
     M2^{rd - \delta(i-r)} & r\leq i < \left(1 + \frac{d}{2s}\right)r\\
     0 & i \geq \left(1 + \frac{d}{2s}\right)r.
     \end{cases}
 \end{equation}
 For simplicity, we allow the $n_i$ to not necessarily be integers for now. We must show two things. First, that
 \begin{equation}\label{entropy-proof-condition-1}
     \sum_{i=0}^\infty n_i \lesssim 2^{rd},
 \end{equation}
 which will ensure that $\sum_{i=1}^\infty n_i \leq n = K2^{rd}$ for sufficiently large $K$. Second, we must use the bound \eqref{partial-dictionary-entropy-bound} to show that
 \begin{equation}\label{entropy-proof-condition-2}
     \sum_{i=0}^\infty \epsilon_{n_i}(B_1(\mathbb{D}_i)) \lesssim 2^{-\left(s+\frac{d}{2}\right)r}.
 \end{equation}
 First, we calculate
 \begin{equation}
 \begin{split}
     \sum_{i=1}^\infty n_i &= M\left(\left(s+\frac{d}{2}+\delta\right)\sum_{i=1}^{r-1}(r-i+1)2^{di} + 2^{rd}\sum_{i=r}^{\left(1 + \frac{d}{2s}\right)r}2^{- \delta(i-r)} \right)\\
     &\leq M2^{rd}\left(\left(s+\frac{d}{2}+\delta\right)\sum_{i=1}^\infty (i+1)2^{-di} + \sum_{i=0}^\infty 2^{-\delta i}\right)\\
     &\lesssim 2^{rd}.
\end{split}
 \end{equation}
 Next, we note that if $i < r$, then $(s+\frac{d}{2} + \delta)(r-i+1) > d \geq 1$ and so for the first $r$ indicies $0\leq i < r$, the last branch of \eqref{partial-dictionary-entropy-bound} is taken. This gives
 \begin{equation}
     \sum_{i=0}^{r-1} \epsilon_{n_i}(B_1(\mathbb{D}_i)) \lesssim 2^{\left(s + \frac{d}{2}\right)(r+1)}\sum_{i=0}^{r-1} 2^{-\delta(r-i+1)} \leq 2^{\left(s + \frac{d}{2}\right)(r+1)}\sum_{i=1}^{\infty} 2^{-\delta i} \lesssim 2^{\left(s + \frac{d}{2}\right)r}.
 \end{equation}
 When $r \leq i < \left(1+\frac{d}{2s}\right)r$, the middle branch in \eqref{partial-dictionary-entropy-bound} is taken, which gives
 \begin{equation}
 \begin{split}
     \sum_{i=r}^{\left(1+\frac{d}{2s}\right)r} \epsilon_{n_i}(B_1(\mathbb{D}_i)) &\lesssim 2^{-\left(s+\frac{d}{2}\right)r}\sum_{i=r}^{\left(1+\frac{d}{2s}\right)r} 2^{-(s - \delta)(i-r)}\sqrt{1+(d-\delta)(i-r)}\\ &\lesssim 2^{-\left(s+\frac{d}{2}\right)r}\sum_{i=0}^\infty 2^{-(s - \delta)i}\sqrt{1+i} \lesssim 2^{-\left(s+\frac{d}{2}\right)r},
\end{split}
 \end{equation}
 as long as $\delta$ is chosen to be less than $s$. If $i \geq \left(1+\frac{d}{2s}\right)r$, then the first branch of \eqref{partial-dictionary-entropy-bound} is taken and we calculate
 \begin{equation}
     \sum_{i=\left(1+\frac{d}{2s}\right)r}^{\infty} \epsilon_{n_i}(B_1(\mathbb{D}_i)) \lesssim 2^{-s\left(1+\frac{d}{2s}\right)r}\sum_{i=0}^\infty 2^{-si}\lesssim 2^{-\left(s+\frac{d}{2}\right)r}.
 \end{equation}
 Finally, since the $n_i$ must be chosen to be integers, we replace $n_i$ by $\lceil n_i\rceil$. 
 Since the right hand side of \eqref{partial-dictionary-entropy-bound} is a decreasing function of $m$, this can only reduce our bound on $\sum_{i=0}^\infty \epsilon_{n_i}(B_1(\mathbb{D}_i))$. Further, since at most $\left(1 + \frac{d}{2}\right)r$ of the $n_i$'s are non-zero, the sum $\sum_{i=0}^\infty n_i$ can increase by at most $\left(1 + \frac{d}{2}\right)r \leq 2^{rd}$. Thus, after making this change conditions \eqref{entropy-proof-condition-1} and \eqref{entropy-proof-condition-2} will still be satisfied, which completes the proof.
 \end{proof}

\subsection{Kolmogorov $n$-width bounds for smoothly parameterized dictionaries}\label{entropy-n-widths-section}
Next, we bound the Kolmogorov $n$-widths of $B_1(\mathbb{D})$ for smoothly parameterized dictionaries $\mathbb{D}$. We have the following theorem.
\begin{theorem}\label{kolmogorov-upper-bound}
 For $s > 0$ and $X$ a Banach space, suppose that $\mathcal{M}$ is a compact $d$-dimensional manifold, $\mathcal{P}:\mathcal{M}\rightarrow X$ is of smoothness class $s$, i.e. $\mathcal{P}\in \lip_\infty^\mathcal{M}(s, X)$, and $\mathbb{D}\subset \mathcal{P}(\mathcal{M})$. Then we have the bound
 \begin{equation}
  d_n(B_1(\mathbb{D}))_X \lesssim n^{-\frac{s}{d}}.
 \end{equation}
 Here the implied constant depends only upon $s$, $d$, and the parameterization map $\mathcal{P}$.

\end{theorem}
As a corollary, we obtain an upper bound on the Kolmogorov widths of $B_1(\mathbb{P}_k^d)$ in $L^p(\Omega)$ of $O(n^{-\frac{pk+1}{pd}})$ for $1 \leq p < \infty$.
\begin{proof}
 For any subspace $V\in X$, the distance map $d(x,V) = \inf_{y\in V} \|x - y\|_X$ is a convex function of $x$. This means that the Kolmogorov $n$-widths are invariant under taking convex hulls, i.e.
 \begin{equation}
     d_n(B_1(\mathbb{D}))_X = d_n(\mathbb{D})_X.
 \end{equation}
 Thus it suffices to bound the $n$-widths of the dictionary $\mathbb{D}$. 
 
 We use Lemma \ref{image-of-union-of-cubes-lemma} to obtain a collection of maps $\mathcal{P}_1,...,\mathcal{P}_T:C\rightarrow X$, where $C = [0,1]^d$ is the unit cube, such that $\mathbb{D}= \cup_{i=1}^T \mathcal{P}_i(C)$.
 Set $\mathbb{D}_i = \mathcal{P}_i(C)$. 
 
 Now let $n \geq 1$ be a fixed integer. We proceed to subdivide the cube $C$ into $n$ subcubes $C_1,...,C_n$ of diameter $O(n^{-\frac{1}{d}})$. Further, let $s = k + \alpha$ with $k$ an integer and $\alpha\in (0,1]$. Let $M=\binom{d+k}{k}$ and choose interpolation points $x_1,...,x_M$ which are unisolvent for the space of polynomials of degree at most $k$. Denote by $x^r_1,...,x^r_M$ the images of these interpolation points in the cube $C_r$ and by $l^r_1,...,l^r_M$ the corresponding Lagrange polynomials. Consider the space
 \begin{equation}
     V_n = \text{span}\{\mathcal{P}_i(x_j^r),~i=1,...,T,j=1,...,M,~r=1,...,n\},
 \end{equation}
 which satisfies $\dim(V_n) \leq TMn$. Given any $d\in \mathbb{D}$, by definition $d = \mathcal{P}_i(z)$ for some $1\leq i\leq T$ and $z\in C_r$ for some $1\leq r\leq n$. Consider the interpolated value
 \begin{equation}
  \mathcal{P}_{i,r,I}(z) = \sum_{i=1}^{M} \mathcal{P}_i(x_i^r)l_i^r(z)\in V_n.
 \end{equation}
 The Banach space Bramble-Hilbert Lemma \ref{polynomial-interpolation-lemma} implies that
 \begin{equation}
     \|d - \mathcal{P}_{i,r,I}(z)\|_X \leq C(k,d)|\mathcal{P}|_{\lip(s,X)}[\diam(C_r)]^s\Lambda_k^d(C_r,x^r_1,...,x^r_M) \lesssim n^{-\frac{s}{d}},
 \end{equation}
 where the implied constant is independent of $n$ since by Lemma \ref{scaling-argument-lemma} the Lebesgue constant $\Lambda_k^d(C_r,x^r_1,...,x^r_M) = \Lambda_k^d(C,x_1,...,x_M)$ is independent of $n$. This means that
 \begin{equation}
     d_{TMn}(\mathbb{D})_X \leq \sup_{d\in \mathbb{D})}\inf_{y\in V_n} \|d - y\|_X \lesssim n^{-\frac{s}{d}},
 \end{equation}
 which completes the proof since $T$ and $M$ are fixed constants independent of $n$.

\end{proof}

\subsection{Gelfand numbers of smoothly parameterized dictionaries}\label{gelfand-numbers-subsection}

Finally, we consider the Gelfand numbers of smoothly parameterized dictionaries $\mathbb{D}$. Denote by $\ell^1(\mathbb{D})$ the Banach space of absolutely summable functions on the dictionary $\mathbb{D}$, i.e.
\begin{equation}
    \ell^1(\mathbb{D}) = \{f:\mathbb{D}\rightarrow \mathbb{R},~\|f\|_{\ell^1(\mathbb{D})} < \infty\},
\end{equation}
where the norm $\|f\|_{\ell^1(\mathbb{D})}$ is given by
\begin{equation}
    \|f\|_{\ell^1(\mathbb{D})} = \sup_{\mathbb{D}_n\subset \mathbb{D}}\sum_{d\in \mathbb{D}_n} |f(d)|.
\end{equation}
Here the supremum $\mathbb{D}_n$ is over all finite subsets of the dictionary $\mathbb{D}$. Define the evaluation map $\mathcal{T}_\mathbb{D}:\ell^1(\mathbb{D})\rightarrow X$ by 
\begin{equation}\label{evaluation-map-definition}
    \mathcal{T}_\mathbb{D} = \sum_{d\in \mathbb{D}} f(d)d.
\end{equation}
It is easy to see that if $\|f\|_{\ell^1(\mathbb{D})} < \infty$ and the dictionary $\mathbb{D}$ is uniformly bounded, then $f$ is non-zero for at most countably many dictionary elements $d$ and the sum in \eqref{evaluation-map-definition} converges absolutely.

For an operator $T:X\rightarrow Y$ between two Banach spaces $X$ and $Y$, we define the Gelfand numbers of the operator $T$ by
\begin{equation}
 c_n(T) = \inf_{U_n\subset X}\|T|_{U_n}\|,
\end{equation}
where the infemum is taken over all closed subspaces of codimension $n$ (see \cite{pietsch1980operator}, Section 11.5). The Gelfand numbers of the convex hull of a dictionary $\mathbb{D}$ are defined to be $c_n(\mathcal{T}_\mathbb{D})$ \cite{carl2013gelfand,carl1988gelfand,carl1999metric}.

We have the following result, which generalizes the results from \cite{carl2013gelfand,carl1999metric} to the case to smoothly parameterized dictionaries.
\begin{theorem}\label{gelfand-width-theorem}
 Let $s > 0$ and $X$ a Hilbert space. Suppose that $\mathcal{M}$ is a compact $d$-dimensional smooth manifold, $\mathcal{P}\in \lip_\infty^\mathcal{M}(s, X)$, and the dictionary $\mathbb{D}\subset \mathcal{P}(\mathcal{M})$. Then
 \begin{equation}\label{entropy-bound-equation}
 c_n(\mathcal{T}_\mathbb{D}) \lesssim n^{-\frac{1}{2} - \frac{s}{d}},
 \end{equation}
 where the implied constants are independent of $n$.
\end{theorem}
The proof of Theorem \ref{gelfand-width-theorem} is analogous to the proof of the entropy bound Theorem \ref{entropy-upper-bound-theorem} and for the sake of brevity we leave these details to the reader. The main difference is that Lemmas \ref{triangle-inequality-entropy-lemma} and \ref{carls-lemma} are replaced by the following three results concerning Gelfand numbers of operators. The last result, Theorem \ref{discrete-gelfand-bound} is a rather deep theorem of Carl and Pajor \cite{carl1988gelfand}.

\begin{lemma}[Theorem 11.8.2 in \cite{pietsch1980operator}]\label{triangle-inequality-gelfand}
 Let $S,T:X\rightarrow Y$. Then for any $0 \leq m\leq n$ we have
 \begin{equation}
  c_n(S+T) \leq c_m(S) + c_{n-m}(T).
 \end{equation}
\end{lemma}
\begin{proof}
 Let $U_m, U_{n-m}\subset X$ be codimension $m$ and $n-m$ subspaces of $X$, respectively, such that
 \begin{equation}
     \|S|_{U_m}\| \leq c_m(S),~\|T|_{U_{n-m}}\| \leq c_{n-m}(T).
 \end{equation}
 If the infimum in the definition of the Gelfand numbers is not achieved a standard limiting argument can be used here. Set $U_n = U_m \cap U_{n-m}$, which is a subspace of codimension at most $n$. Then we have
 \begin{equation}
     \|(S+T)|_{U_n}\| \leq \|S|_{U_n}\| + \|T|_{U_n}\| \leq \|S|_{U_m}\| + \|T|_{U_{n-m}}\| = c_m(S) + c_{n-m}(T).
 \end{equation}
\end{proof}
\begin{lemma}\label{gelfand-width-composition-lemma}
 Let $S:X\rightarrow Y$ and $T:Y\rightarrow Z$. Then we have
 \begin{equation}
     c_n(ST) = \|S\|c_n(T).
 \end{equation}
\end{lemma}
\begin{proof}
 Let $U_n\subset Y$ be a subspace of codimension $n$. Then $V_n := S^{-1}(U_n)\subset X$ is a subspace of codimension at most $n$. Then $\|ST|_{V_n}\| \leq \|S\|\|T_{U_n}\|$ and the result follows. 
\end{proof}

\begin{theorem}\label{discrete-gelfand-bound}[Theorem 2.2 in \cite{carl1988gelfand}]
 Let $T:\ell_1^n\rightarrow H$ be a bounded linear operator where $H$ is a Hilbert space. Then we have
 \begin{equation}
  c_m(T) \lesssim \begin{cases} 
         \|T\| & m=0 \\
          \sqrt{1+\log{\frac{n}{m}}}m^{-\frac{1}{2}}\|T\| & 1\leq m < n \\
          0 & m\geq n.
       \end{cases}
 \end{equation}
 Note the implied constant here is absolute.
\end{theorem}

Given the bound in Theorem \ref{gelfand-width-theorem}, a natural question is how the Gelfand numbers of the dictionary $\mathbb{D}$ are related to the Gelfand widths $d^n(B_1(\mathbb{D})$ of the convex hull of $\mathbb{D}$, which measure how efficiently functions from $B_1(\mathbb{D})$ can be recovered from linear measurements. By definition
\begin{equation}
    B_1(\mathbb{D}) = \overline{\mathcal{T}_\mathbb{D}(B_{\ell^1(\mathbb{D})})},
\end{equation}
where $B_{\ell^1(\mathbb{D})}$ denote the unit ball in $\ell^1(\mathbb{D})$. Thus this question is a special case of the general question of how the Gelfand numbers $c_n(T)$ are related to the Gelfand widths $d^n(T(B_X))$ for a general operator $T:X\rightarrow Y$ between two Banach spaces $X$ and $Y$, where $B_X$ is the unit ball of $X$.

Expanding out the definitions, we have
\begin{equation}\label{gelfand-number-def}
    c_n(T) = \inf_{\xi_1,...,\xi_n\in X^*} \sup\{\|T(x)\|_Y:~x\in B_X,~\xi_i(x) = 0,~i=1,...,n\},
\end{equation}
and on the other hand
\begin{equation}
    d^n(T(B_X)) = \inf_{\xi_1,...,\xi_n\in Y^*} \sup\{\|T(x)\|_y:~x\in B_X,~\xi_i(T(x)) = 0,~i=1,...,n\}.
\end{equation}
Since $\xi_i(T(x)) = (T^*\xi_i)(x)$ and $T^*\xi_i\in X^*$, we see that the infimum in \eqref{gelfand-number-def} is over a larger set, so that $c_n(T) \leq d^n(B_X)$. However, if $T$ is not injective, the map $T^*$ will not be surjective and the infemum in \eqref{gelfand-number-def} is over a strictly larger set. In such a situation it is possible that strict inequality holds \cite{edmunds2013gelfand,pietsch2007history}, i.e. that $c_n(T) < d^n(B_X)$.

Equality of the Gelfand numbers and widths has been established under certain conditions on the operator $T$, see \cite{edmunds2013gelfand} for instance. However, these conditions all suppose the injectivity of the operator $T$. In fact, when $T$ is not injective the typical situation is that $c_n(T) < d^n(B_X)$. To illustrate this, we give the following example of a small finite dictionary $\mathbb{D}$ for which $c_n(\mathcal{T}_\mathbb{D}) < d^n(B_1(\mathbb{D}))$. This shows that in general there isn't much hope to bound the Gelfand widths $d^n(\mathbb{D})$ using Theorem \ref{gelfand-width-theorem} and leaves open the problem of developing techniques for bounding $d^n(B_1(\mathbb{D}))$ in the case where the evaluation map $\mathcal{T}_\mathbb{D}$ is not injective.
\begin{proposition}
Consider the following dictionary $\mathbb{D}\subset \mathbb{R}^3$
\begin{equation}
    \mathbb{D} = \left\{\begin{pmatrix}
           1 \\
           0 \\
           0
         \end{pmatrix},\begin{pmatrix}
           1/2 \\
           \sqrt{3}/2 \\
           0
         \end{pmatrix},\begin{pmatrix}
           -1/2 \\
           \sqrt{3}/2 \\
           0
         \end{pmatrix},
         \begin{pmatrix}
           0 \\
           0 \\
           \sqrt{3}
         \end{pmatrix}\right\}.
\end{equation}
Then $c_1(\mathcal{T}_\mathbb{D}) < d^1(B_1(\mathbb{D}))$.
\end{proposition}
\begin{proof}
 Note that for this dictionary the map $\mathcal{T}_\mathbb{D}:\ell_1^4\rightarrow \ell_2^3$ is given by the following matrix
 \begin{equation}
     \begin{pmatrix}
      1 & 1/2 & -1/2 & 0\\
      0 & \sqrt{3}/2 & \sqrt{3}/2 & 0\\
      0 & 0 & 0 & \sqrt{3}
     \end{pmatrix}.
 \end{equation}
 Consider the subspace $U_1\subset \ell_1^4$ defined by $\xi\cdot x = 0$ where 
 \begin{equation}
     \xi = \begin{pmatrix}
      1 \\
      -1 \\
      1 \\
      3
     \end{pmatrix}.
 \end{equation}
 In order to calculate $\|\mathcal{T}_\mathbb{D}|_{U_1}\|$, we determine the extreme points of the intersection 
 \begin{equation}
     U_1\cap B_{\ell_1^4} = \{x\in \mathbb{R}^4,~\xi\cdot x = 0,~|\xi|_1 \leq 1\}.
 \end{equation}
 The unit ball of $\ell_1^4$ is the convex hull of $\{\pm e_1,...,\pm e_4\}$ and the extreme points of $U_1\cap B_{\ell_1^4}$ must be a linear combinations of at most two of these vectors. Using the form of $\xi$, these extreme points are
 \begin{equation}
     E := \left\{\frac{1}{2}(\pm e_1 \pm e_2), \frac{1}{2}(\pm e_1\mp e_3), \pm \frac{3}{4}e_1\mp \frac{1}{4}e_4, \frac{1}{2}(\pm e_2\pm e_3), \pm \frac{3}{4}e_2 \pm \frac{1}{4}e_4, \pm \frac{3}{4}e_3\mp \frac{1}{4}e_4\right\}.
 \end{equation}
 The norm $\|\mathcal{T}_\mathbb{D}|_{U_1}\|$ is equal to the maximum value of $\|\mathcal{T}_\mathbb{D}(x)\|_2$ for $e\in E$ and a straightforward calculation yields
 \begin{equation}
     \|\mathcal{T}_\mathbb{D}|_{U_1}\| = \frac{\sqrt{3}}{2}.
 \end{equation}
 Thus we get $c_1(\mathcal{T}_\mathbb{D}) \leq \sqrt{3}/2$.
 
 Next, we will show that $d^1(B_1(\mathbb{D})) > \sqrt{3}/2$. Note that the shape of $B_1(\mathbb{D})$ is a hexagonal bipyramid. Suppose that there exists a plane $U\subset \mathbb{R}^3$ such that $U\cap B_1(\mathbb{D})$ is contained in a ball of radius $\sqrt{3}/2$. Consider $U \cap \text{span}(e_1,e_2)$. This intersection cannot contain points of the hexagonal base of $B_1(\mathbb{D})$ which are longer than $\sqrt{3}/2$. Thus $U \cap \text{span}(e_1,e_2)$ must be a line connecting the midpoints of two opposite side of this hexagon. So we can assume without loss of generality that
 \begin{equation}
     U \cap \text{span}(e_1,e_2) = \text{span}\left\{\begin{pmatrix}
      \sqrt{3}/2\\
      -1/2\\
      0
     \end{pmatrix}\right\}.
 \end{equation}
 This in turn implies that $U$ must intersect the line segments connecting
 \begin{equation}\label{line-segments}
     \begin{pmatrix}
      1\\
      0\\
      0
     \end{pmatrix},~\begin{pmatrix}
      1/2\\
      \sqrt{3}/2\\
      0
     \end{pmatrix},~\begin{pmatrix}
      -1/2\\
      \sqrt{3}/2\\
      0
     \end{pmatrix}~\text{to either}~\begin{pmatrix}
      0\\0\\\sqrt{3}
     \end{pmatrix}~\text{or}~\begin{pmatrix}
      0\\0\\-\sqrt{3}
     \end{pmatrix}.
 \end{equation}
 By reflecting, we may assume without loss of generality that the former occurs. However, each of these line segments contains a unique point with length at most $\sqrt{3}/2$, which are given by
 \begin{equation}\label{points-equation}
     \begin{pmatrix}
      3/4\\
      0\\
      \sqrt{4}/4
     \end{pmatrix},\begin{pmatrix}
      3/8\\
      3\sqrt{3}/8\\
      \sqrt{4}/4
     \end{pmatrix},\begin{pmatrix}
      -3/8\\
      3\sqrt{3}/8\\
      \sqrt{4}/4
     \end{pmatrix},
 \end{equation}
 respectively. Since $B_1(\mathbb{D})$ contains each of the line segments in \eqref{line-segments} and $U\cap B_1(\mathbb{D})$ is contained in a ball of radius $\sqrt{3}/2$, this implies that $U$ must contain each of the points in \eqref{points-equation}. Finally, we note that the points in \eqref{points-equation} are linearly independent and thus cannot all be contained in the two dimensional subspace $U$. This contradiction shows that $d^1(B_1(\mathbb{D})) > \sqrt{3}/2$ and completes the proof.
\end{proof}

\section{Lower Bounds for Dictionaries of Ridge Functions}\label{lower-bounds-section}
In this section, we consider lower bounds on the metric entropy, Kolmogorov, and Bernstein $n$-widths of convex subsets $A$ of $L^2(\Omega)$. We show that if $A$ contains a certain class of ridge functions, then these quantities must be bounded below. We will apply this result to lower bound the entropy and $n$-widths of variation spaces corresponding to shallow neural networks.

Our method works by constructing a large collection of nearly orthogonal vectors in $A$ and then obtaining lower bounds by noting that $A$ must contain the convex hull of these vectors. We begin with some Lemmas lower bounding the entropy, Kolmogorov, and Bernstein $n$-widths of such a convex hull. This idea has been used to lower bound the entropy in \cite{makovoz1996random,klusowski2018approximation}, yet these authors did not find as large a collection of nearly orthogonal vectors and obtained suboptimal bounds as a result.

\begin{lemma}\label{lower-eigenvalue-lemma}
 Let $H$ be a Hilbert space and $A\subset H$ a convex and symmetric set. Suppose that $g_1,...,g_n\subset A$. Then
 \begin{equation}\label{lemma-lower-bound}
  \epsilon_{n}(A)\geq \frac{1}{2}\sqrt{\frac{\lambda_{min}}{n}},~b_n(A)\geq \sqrt{\frac{\lambda_{min}}{n}}
 \end{equation}
 where $\lambda_{min}$ is the smallest eigenvalue of the Gram matrix $G$ defined by $G_{ij} = \langle g_i,g_j\rangle_H$.
\end{lemma}
\begin{proof}
 Consider a maximal set of points $x_1,...,x_N\in b_1^n(0,1):=\{x\in \mathbb{R}^n:~|x|_1\leq 1\}$ in the $\ell^1$-unit ball satisfying $|x_i - x_j| \geq \frac{1}{2}$ for each $i\neq j$. We claim that $N \geq 2^n$. Indeed, if the set $\{x_i\}_{i=1}^N$ is maximal, then the balls 
 $$b^n_1(x_i,1/2) = \left\{x\in \mathbb{R}^n:~|x-x_i|_1\leq \frac{1}{2}\right\}$$
 must cover the ball $b_1^n(0,1)$. This implies that
 \begin{equation}
  \sum_{i=1}^N |b^n_1(x_i,1/2)| \geq |b_1^n(0,1)|.
 \end{equation}
 Since we obviously have $|b^n_1(x_i,1/2)| = (1/2)^n|b_1^n(0,1)|$ for each $i$, it follows that $N \geq 2^n$.
 
Consider the collection of elements $f_1,...,f_N\in H$ defined by
 \begin{equation}
  f_i = \sum_{k=1}^nx^k_ig_k.
 \end{equation}
 Since $A$ is symmetric and convex, we have $f_i\in A$ for each $i=1,...,N$. Moreover, if $i\neq j$, then
 \begin{equation}
  \|f_i-f_j\|^2_H = v^T_{ij}Gv_{ij},
 \end{equation}
 where $v_{ij} = x_i - x_j$. Since $|x_i - x_j|_1 \geq \frac{1}{2}$, it follows from H\"older's inequality that $|v_{ij}|^2_2 \geq \frac{1}{4n}$. From the eigenvalues of $G$ we then see that $\|f_i-f_j\|^2_H \geq \frac{\lambda_{min}}{4n}$ for all $i\neq j$. This gives the entropy lower bound \eqref{lemma-lower-bound}.
 
 To lower bound the Bernstein-widths, we note that if $g_1,...,g_n$ are linearly dependent, then $\lambda_{min} = 0$ and there is noting to prove. On the other hand, consider the linear subspace $V_n$ spanned by the $g_i$. Then $V_n\cap A$ contains the convex hull of $g_1,...g_n$ and so for every $x\in \partial(A\cap V_n)$ we have
 \begin{equation}
  \|x\|^2_H \geq \inf_{\|a\|_1 = 1} \left\|\sum_{i=1}^n a_ig_i\right\|^2_H \geq \lambda_{min}n^{-1},
 \end{equation}
 since $\|a\|_1 = 1$ implies that $\|a\|_2^2 \geq n^{-1}$. This completes the bound on the Bernstein widths.
\end{proof}

This Lemma can be applied to sequences of almost orthogonal vectors to obtain Lemma 3 from \cite{makovoz1996random}, which we state here as a corollary for completeness.
\begin{corollary}\label{entropy-lower-bound-corollary}
 Let $H$ be a Hilbert space and $A\subset H$ a convex and symmetric set. Suppose that $g_1,...,g_n\subset A$ and the the $g_i$ are almost orthogonal in the sense that for all $i = 1,...,n$,
 \begin{equation}\label{diagonal-dominant}
  \sum_{j\neq i}|\langle g_i,g_j\rangle_H| \leq \frac{1}{2}\|g_i\|_H^2.
 \end{equation}
 Then
 \begin{equation}
  \epsilon_{n}(A)\geq \frac{\min_i \|g_i\|_H}{\sqrt{8n}},~b_{n}(A)\geq \frac{\min_i \|g_i\|_H}{\sqrt{2n}}
 \end{equation}
\end{corollary}
\begin{proof}
 This follows from Lemma \ref{lower-eigenvalue-lemma} if we can show that the Gram matrix $G$ satisfies
 \begin{equation}
  \lambda_{min}(G) \geq \frac{1}{2}\min_i \|g_i\|^2_H.
 \end{equation}
 This follows immediately from the diagonal dominance condition \ref{diagonal-dominant} and the Gerschgorin circle theorem (see the proof in \cite{makovoz1996random} for details).
\end{proof}

In order to lower bound the Kolmogorov $n$-widths, we will need the following Lemma, which generalizes Lemma 6 in \cite{barron1993universal} to almost orthogonal sets, which satisfy a stronger notion of almost orthogonality than that in Corollary \ref{entropy-lower-bound-corollary}.
\begin{lemma}\label{kolmogorov-lower-bound-lemma}
  Let $H$ be a Hilbert space and $A\subset H$ a convex and symmetric set. Suppose that $g_1,...,g_{2n}\subset A$ and the the $g_i$ are almost orthogonal in the sense that for all $i = 1,...,2n$,
 \begin{equation}\label{diagonal-dominant-kolmogorov}
  \sum_{j\neq i}|\langle g_i,g_j\rangle_H| \leq \frac{1}{2}\min_j \|g_j\|_H^2.
 \end{equation}
 Then
 \begin{equation}
  d_{n}(A)\geq \frac{1}{2}\min_j \|g_j\|_H.
 \end{equation}
\end{lemma}
\begin{proof}
 By scaling down the $g_j$ if necessary, we may assume that $\|g_j\|_H = \min_i \|g_i\|_H$ for all $j$. This follows since the rescaled vectors will clearly be in $A$ (due to symmetry) and condition \eqref{diagonal-dominant-kolmogorov} will still be satisfied (since the left hand side can only decrease while the right hand side is unchanged). We can further assume without loss of generality that $\|g_j\|_H = 1$ for all $j=1,...,2n$.
 
 Let $V_n$ be an $n$-dimensional subspace of $H$ with orthonormal basis $e_1,...,e_n$. For each index $i=1,...,n$ we will have
 \begin{equation}
  \sum_{j=1}^{2n}|\langle e_i,g_j\rangle|^2 \leq \lambda_{max}(G),
 \end{equation}
 where $G$ is the Gram matrix of the $g_j$. Using the Gerschgorin circle theorem and condition \eqref{diagonal-dominant-kolmogorov}, we get $\lambda_{max}(G) \leq \frac{3}{2}$. Summing over $i$ and switching the order of summation, we get
 \begin{equation}
  \sum_{j=1}^{2n}\sum_{i=1}^n|\langle e_i,g_j\rangle|^2 = \sum_{i=1}^n\sum_{j=1}^{2n}|\langle e_i,g_j\rangle|^2 \leq \frac{3}{2}n.
 \end{equation}
 From this, we see that there must exist an index $j$, such that
 \begin{equation}
  \sum_{i=1}^n|\langle e_i,g_j\rangle|^2 \leq \frac{3}{4}.
 \end{equation}
 But this means that the projection of $g_j$ onto $V_n$ has norm at most $\frac{3}{4}$, so that $d(g_j,V_n) \geq \frac{1}{2}$. Since this bound holds for some $j$ for any subspace $V_n$ of dimension $n$, we get the desired lower bound.

\end{proof}

Using the relationship between the $\mathcal{K}_1(\mathbb{P}_k^d)$-norm and the spectral Barron norm \eqref{barron-spectral-barron-bound}, we obtain that
\begin{equation}
 \|f_\xi\|_{\mathcal{K}_1(\mathbb{P}_k^d)} \lesssim 1,
\end{equation}
where $f_\xi(x) = (1 + |\xi|)^{-(k+1)}e^{2\pi \iu \xi\cdot x}$. In other words the space $\mathcal{K}_1(\mathbb{P}_k^d)$ contains appropriately rescaled frequencies.

By considering the collection of functions $f_\xi$ for $\xi\in \mathbb{Z}^d$ with $|\xi|_\infty \leq R$, we can make the $f_\xi$ orthogonal on $[0,1]^d$. Applying Lemmas \ref{entropy-lower-bound-corollary} and \ref{diagonal-dominant-kolmogorov} with this set of functions yields the bounds
\begin{equation}
    \epsilon_n(B_1(\mathbb{P}_k^d))_{L^2([0,1]^d)} \gtrsim_{k,d} n^{-\frac{1}{2}-\frac{k+1}{d}},~b_n(B_1(\mathbb{P}_k^d))_{L^2([0,1]^d)} \gtrsim_{k,d}n^{-\frac{1}{2}-\frac{k+1}{d}},~d_n(B_1(\mathbb{P}_k^d))_{L^2([0,1]^d)} \gtrsim_{k,d} n^{-\frac{k+1}{d}}.
\end{equation}
This argument, which is essentially using the fact that the spectral Barron norm bounds the $\mathcal{K}_1(\mathbb{P}_k^d)$-norm, was used in \cite{makovoz1996random,klusowski2018approximation} to obtain a lower bound on the metric entropy $\epsilon_n(B_1(\mathbb{P}_k^d))_{L^2([0,1]^d)}$.

However, it is known that $\mathcal{B}_{k+1}\subsetneq \mathcal{K}_1(\mathbb{P}_k^d)$, in other words that the spectral Barron space is strictly smaller than the variation space $\mathcal{K}_1(\mathbb{P}_k^d)$ \cite{ma2019barron,wojtowytsch2020representation}. Consequently it should be possible to obtain a better lower bound on $\epsilon_n(B_1(\mathbb{P}_k^d))_{L^2}$, $b_n(B_1(\mathbb{P}_k^d))_{L^2}$, and $d_n(B_1(\mathbb{P}_k^d))_{L^2}$, which would precisely quantify the gap between the spectral Barron space and $\mathcal{K}_1(\mathbb{P}_k^d)$.

The first such improved lower bound on $\epsilon_n(B_1(\mathbb{P}_k^d))_{L^2}$ was obtained by Makovoz \cite{makovoz1996random} in the case $k=0$, $d=2$, and it was conjectured that an improved lower bound holds more generally. We settle this conjecture by deriving an improved lower bound for all $k \geq 0$ and $d\geq 2$ and removing a logarithm from Makovoz's original bound. In addition, we also derive an improved lower bound on the Kolmogorov $n$-widths $d_n(B_1(\mathbb{P}_k^d))_{L^2}$ and a bound on the Bernstein widths $b_n(B_1(\mathbb{P}_k^d))_{L^2}$.

\begin{theorem}\label{lower-bound-theorem}
 Let $d \geq 2$, $k\geq 0$, and denote the unit ball in $\mathbb{R}^d$ by
 \begin{equation}
     B_1^d = \{x\in \mathbb{R}^d,~|x|_2\leq 1\}.
 \end{equation}
 Let $A\subset L^2(B_1^d)$ be a convex and symmetric set. Suppose that for every profile $\phi\in C_c^\infty([-2,2])$ such that $\|\phi^{(k+1)}\|_{L^1(\mathbb{R})}\leq 1$, and any direction $\omega\in S^{d-1}$, the ridge function $\phi(\omega\cdot x)\in L^2(B_1^d)$ satisfies
 \begin{equation}
  \phi(\omega\cdot x)\in A.
 \end{equation}
 Then
 \begin{equation}
  \epsilon_n(A)_{L^2(B_1^d)} \gtrsim_{k,d} n^{-\frac{1}{2}-\frac{2k+1}{2d}},~b_n(A)_{L^2(B_1^d)} \gtrsim_{k,d}n^{-\frac{1}{2}-\frac{2k+1}{2d}},~d_n(A)_{L^2(B_1^d)} \gtrsim_{k,d} n^{-\frac{2k+1}{2d}}.
 \end{equation}

\end{theorem}
The argument we give here adapts the argument in the proof of Theorem 4 in \cite{makovoz1996random}. A careful analysis allows us extend the result to higher dimensions and remove a logarithmic factor. The key is to consider profiles $\phi$ whose higher order moments vanish in combination with a weighted $L^2$-norm with a Bochner-Riesz type weight. 

Before we give the proof, we observe that the Peano kernel formula
\begin{equation}
 \phi(x) = \frac{1}{k!}\int_{-2}^2 \phi^{(k+1)}(t)[\max(0,x-t)]^kdt = \frac{1}{k!}\int_{-2}^2 \phi^{(k+1)}(t)\sigma_k(0,x-t)dt,
\end{equation}
which holds for all $\phi\in C_c^\infty([-2,2])$, implies that for a constant $C = C(k,d)$, the unit ball $CB_1(\mathbb{P}^d_k)$ satisfies the conditions of Theorem \ref{lower-bound-theorem}. Combining this with the fact that any bounded domain $\Omega$ is contained in a large enough ball yields the result given in the introduction:
\begin{theorem}\label{relu-k-lower-bound-corollary}
 Let $d \geq 2$ and $\Omega\subset \mathbb{R}^d$ a bounded domain. Then
 \begin{equation}
  \epsilon_n(B_1(\mathbb{P}^d_k))_{L^2(\Omega)} \gtrsim_{k,d} n^{-\frac{1}{2}-\frac{2k+1}{2d}},~b_n(B_1(\mathbb{P}^d_k))_{L^2(\Omega)} \gtrsim_{k,d} n^{-\frac{1}{2}-\frac{2k+1}{2d}},~d_n(B_1(\mathbb{P}^d_k))_{L^2(\Omega)} \gtrsim_{k,d} n^{-\frac{2k+1}{2d}}.
 \end{equation}

\end{theorem}

Note that the lower bound for $k=0$ also applies to the variation spaces for networks with more general sigmoidal activation functions as well. This follows by a standard argument which scales the sigmoidal function to approximate a Heaviside activation function \cite{barron1993universal}. In addition, Theorem \ref{lower-bound-theorem} can be applied to more general activation functions as well, for instance the B-spline activation functions $\sigma_{k,B}$, but we do not give the details here.

\begin{proof}[Proof of Theorem \ref{lower-bound-theorem}]
 We introduce the weight $$d\mu = (1-|x|^2)_+^{\frac{d}{2}}dx$$ of Bochner-Riesz type on $B_1^d$ and consider the space $H = L^2(B_1^d,d\mu)$. Since $1-|x|^2 \leq 1$, it follows that
 $\|f\|_H \leq \|f\|_{L^2(\Omega)}$, and so it suffices to lower bound the entropy and $n$-widths of $A$ with respect to the weighted space $H$.
 
 Choose $0\neq \psi\in C^\infty_c([-1,1])$ such that $2d-1$ of its moments vanish, i.e. such that
 \begin{equation}
  \int_{-1}^1 x^r\psi(x)dx = 0,
 \end{equation}
 for $r=0,...,2d-2$. Such a function $\psi$ can easily be obtained by convolving an arbitrary compactly supported function whose moments vanish (such as a Legendre polynomial) with a $C^\infty$ bump function.

 Our assumptions on the set $A$ imply that by scaling $\psi$ appropriately, we can ensure that for $0 < \delta < 1$
 \begin{equation}
   \delta^{k}\psi(\delta^{-1}\omega\cdot x + b)\in A,
  \end{equation}
 for any $\omega\in S^{d-1}$ and $b\in[-\delta^{-1},\delta^{-1}]$. Note that $\psi$, which will be fixed in what follows, depends upon both $d$ and $k$.

 Let $N \geq 1$ be an integer and fix $n = N^{d-1}$ directions $\omega_1,...,\omega_n\in S^{d-1}$ with $\min(|\omega_i - \omega_j|_2, |\omega_i + \omega_j|_2) \gtrsim_d N^{-1}$. This can certainly be done since projective space $P^{d-1} = S^{d-1}/\{\pm\}$ has dimension $d-1$. In particular, if $\omega_1,...,\omega_n$ is a maximal set satisfying $\min(|\omega_i - \omega_j|_2, |\omega_i + \omega_j|_2) \geq cN^{-1}$, then balls of radius $cN^{-1}$ centered at the $\omega_i$ must cover $P^{d-1}$. So we must have $n = \Omega(N^{d-1})$, and by choosing $c$ appropriately we can arrange $n = N^{d-1}$.
 
 Further, let $a \leq \frac{1}{4}$ be a sufficiently small constant to be specified later and consider for $\delta = aN^{-1}$ the collection of functions
 \begin{equation}
  g_{p,l}(x) = \delta^{k}\psi(\delta^{-1}\omega_p \cdot x + 2l)\in A,
 \end{equation}
 for $p=1,...,n$ and $l = -\frac{N}{2},...,\frac{N}{2}$. 
 
 The intuition here is that $g_{p.l}$ is a ridge function which varies in the direction $\omega_p$ and has the compactly supported profile $\psi$ dialated to have width $\delta$ (and scaled appropriately to remain in $A$). The different values of $l$ give different non-overlapping shifts of these functions.  The proof proceeds by checking that the $g_{p,l}$ can be made `nearly orthogonal' by choosing $a$ sufficiently small.
 
 Indeed, we claim that if $a$ is chosen small enough, then the $g_{p,l}$ satisfy the conditions of Lemma \ref{kolmogorov-lower-bound-lemma}, i.e. for each $(p,l)$
 \begin{equation}
  \sum_{(p',l')\neq (p,l)} |\langle g_{p,l}, g_{p',l'}\rangle_H| \leq \frac{1}{2}\min_{(p',l')}\|g_{p',l'}\|^2_H.
 \end{equation}
 This will of course also imply that the weaker conditions of Corollary \ref{entropy-lower-bound-corollary} will be satisfied.
 
 Before giving the detailed calculation, we describe the key ideas. 
 
 If we consider two different directions $\omega_p$ and $\omega_{p'}$, functions $g_{p,l}$ and $g_{p',l}$ will be constant along the $(d-2)$-dimensional subspace orthogonal to both $\omega_p$ and $\omega_{p'}$. Thus the inner product $\langle g_{p,l}, g_{p',l'}\rangle_H$ corresponds to an integral over a circle in the plane spanned by $\omega_p$ and $\omega_{p'}$. The integrand is given by a product of the profile $\psi$ supported in two intersecting strips multiplied by the integral of the Bochner-Riesz weight $d\mu$ along the $(d-2)$-dimensional subspace orthogonal to $\omega_p$ and $\omega_{p'}$. The weight $d\mu$ has been chosen so that when we integrate out this $(d-2)$-dimensional subspace, we will obtain a polynomial which vanishes to a high degree at the boundary of the circle (and is zero outside). This, combined with the high-order vanishing of the profile $\psi$, results in the functions $g_{p,l}$ and $g_{p',l}$ satisfying the required `near orthogonality' bounds. We give the detailed calculations in the following.
 
 We begin by estimating $\|g_{p,l}\|^2_H$, as follows
 \begin{equation}
  \|g_{p,l}\|^2_H = \delta^{2k}\int_{B_1^d} |\psi(\delta^{-1}\omega_p \cdot x + 2l)|^2(1-|x|^2)^{\frac{d}{2}}dx.
 \end{equation}
 We proceed to complete $\omega_p$ to an orthonormal basis of $\mathbb{R}^d$, $b_1 = \omega_p, b_2,...,b_d$ and denote the coordinates of $x$ with respect to this basis by $y_i = x\cdot b_i$. Rewriting the above integral in this new orthonormal basis, we get
 \begin{equation}
 \begin{split}
  \|g_{p,l}\|^2_H &= \delta^{2k}\int_{B_1^d}|\psi(\delta^{-1}y_1 + 2l)|^2\left(1-\sum_{i=1}^d y_i^2\right)^{\frac{d}{2}}dy_1\cdots dy_d \\
  &= \delta^{2k}\int_{-1}^1|\psi(\delta^{-1}y_1 + 2l)|^2 \rho_d(y_1)dy_1,
  \end{split}
 \end{equation}
 where
 \begin{equation}
 \begin{split}
  \rho_d(y) &= \int_0^{\sqrt{1-y^2}} (1-y^2-r^2)^{\frac{d}{2}}r^{d-2}dr\\ 
  &= (1-y^2)^{d-\frac{1}{2}}\int_0^{1} (1-r^2)^{\frac{d}{2}}r^{d-2}dr = K_d(1-y^2)^{d-\frac{1}{2}},
  \end{split}
 \end{equation}
 for a dimension dependent constant $K_d$.

 Further, we change variables, setting $y = \delta^{-1}y_1 + 2l$ and use the fact that $\psi$ is supported in $[-1,1]$, to get
 \begin{equation}
  \|g_{p,l}\|^2_H = K_d\delta^{2k+1}\int_{-1}^1 |\psi(y)|^2 (1-[\delta(y-2l)]^2)^{d-\frac{1}{2}} dy.
 \end{equation}
  Since $|y| \leq 1$ and $|2l| \leq N$, as long as $\delta(N+1) \leq 1/2$, which is guaranteed by $a \leq \frac{1}{4}$, the coordinate $y_1 = \delta(y-2l)$ will satisfy $|y_1| \leq 1/2$. This means that $$(1-[\delta(y-2l)]^2)^{d-\frac{1}{2}} = (1-y_1^2)^{d-\frac{1}{2}} \geq (3/4)^{d-\frac{1}{2}}$$ uniformly in $p,l,N$ and $\delta$, and thus
 \begin{equation}\label{lower-bound}
  \|g_{p,l}\|^2_H \geq K_d(3/4)^{d-\frac{1}{2}}\delta^{2k+1}\int_{-1}^1 |\psi(y)|^2dy \gtrsim_{k,d} \delta^{2k+1}.
 \end{equation}
 
 Next consider $|\langle g_{p,l}, g_{p',l'}\rangle_H|$ for $(p,l)\neq (p',l')$. 
 
 If $p=p'$, then $\omega_p = \omega_{p'}$, but $l\neq l'$. In this case, we easily see that the supports of $g_{p,l}$ and $g_{p,l'}$ are disjoint and so the inner product $\langle g_{p,l}, g_{p',l'}\rangle_H = 0$. 
 
 On the other hand, if $p\neq p'$ we get
 \begin{equation}
  \langle g_{p,l}, g_{p',l'}\rangle_{H} = \delta^{2k}\int_{B_1^d} \psi(\delta^{-1}\omega_p\cdot x + 2l)\psi(\delta^{-1}\omega_{p'}\cdot x + 2l')(1-|x|^2)^{\frac{d}{2}}dx.
 \end{equation}
 Since $p\neq p'$, the vectors $\omega_p$ and $\omega_{p'}$ are linearly independent and we complete them to a basis $b_1 = \omega_p, b_2 = \omega_{p'}, b_3,...,b_d$, where $b_3,...,b_d$ is an orthonormal basis for the subspace orthogonal to $\omega_p$ and $\omega_{p'}$. 
 
 Letting $b_1',b_2',b_3'=b_3,...,b_d'=b_d$ be a dual basis (i.e. satisfying $b_i'\cdot b_j = \delta_{ij}$) and making the change of variables $x = y_1b_1' + \cdots + y_db_d'$ in the above integral, we get
 \begin{equation}\label{inner-product-equation}
  \langle g_{p,l}, g_{p',l'}\rangle_{H} = \delta^{2k}\det(D_{p,p'})^{-\frac{1}{2}} \int_{-\infty}^\infty \int_{-\infty}^\infty \psi(\delta^{-1}y_1+2l)\psi(\delta^{-1}y_2+2l') \gamma_d(|y_1b_1' + y_2b_2'|) dy_1dy_2,
 \end{equation}
 where $D_{p,p'}$ is the Graham matrix of $\omega_1$ and $\omega_2$ (notice that then $D_{p,p'}^{-1}$ is the Graham matrix of $b_1'$ and $b_2'$) and
\begin{equation}
 \begin{split}
  \gamma_d(y) &= \int_0^{\sqrt{1-y^2}} (1-y^2-r^2)^{\frac{d}{2}}r^{d-3}dr\\ 
  &= (1-y^2)_+^{d-1}\int_0^{1} (1-r^2)^{\frac{d}{2}}r^{d-3}dr = K'_d(1-y^2)_+^{d-1},
  \end{split}
 \end{equation}
 for a second dimension dependent constant $K'_d$. (Note that if $d=2$, then the above calculation is not correct, but we still have $\gamma_d(y) =  (1-y^2)_+^{\frac{d}{2}} = (1-y^2)_+^{d-1}$.) We remark that the choice of Bochner-Riesz weight $d\mu = (1 - |x|^2)_+^{\frac{d}{2}}$ was made precisely so that $\gamma_d$ is a piecewise polynomial with continuous derivatives of order $d-2$, which will be important in what follows.
 
 Next, we fix $y_1$ and analyze, as a function of $z$,
 $$\tau_{p,p'}(y_1,z) = \gamma_d(|y_1b_1' + zb_2'|) = K'_d(1-q_{p,p'}(y_1,z))_+^{d-1},$$
 where $q_{p,p'}$ is the quadratic 
 \begin{equation}\label{definition-of-q}
 q_{p,p'}(y_1,z) = (b_1'\cdot b_1')y_1^2-2(b_1'\cdot b_2')y_1z-(b_2'\cdot b_2')z^2,
 \end{equation}
 We observe that, depending upon the value of $y_1$, $\tau_{p,p'}(y_1,z)$ is either identically $0$ or is a piecewise polynomial function of degree $2d-2$ with exactly two break points at the roots $z_1,z_2$ of $q_{p,p'}(y_1,z) = 1$. Furthermore, utilizing Fa\`a di Bruno's formula \cite{di1857note} and the fact that $q_{p,p'}(y_1,\cdot)$ is quadratic, we see that
 \begin{equation}\label{derivative-of-tau}
  \left.\frac{d^k}{dz^k} \tau_{p,p'}(y_1,z)\right|_{z_i} = \sum_{m_1+2m_2=k} \frac{k!}{m_1!m_2!2^{m_2}}f_d^{(m_1+m_2)}(1)\left[\frac{d}{dz}q_{p,p'}(y_1,z)|_{z_i}\right]^{m_1}\left[\frac{d^2}{dz^2}q_{p,p'}(y_1,z)|_{z_i}\right]^{m_2},
 \end{equation}
 where $f_d(x) = (1-x)^{d-1}$. 
 
 Since $f^{(m)}_d(1) = 0$ for all $m \leq d-2$, we see that
 the derivative in \eqref{derivative-of-tau} is equal to $0$ for $0 \leq k\leq d-2$. Thus the function $\tau_{p,p
 }(y_1,\cdot)$ has continuous derivatives up to order $d-2$ at the breakpoints $z_1$ and $z_2$. Moreover, if we consider the derivative of order $k=d-1$, then only the term with $m_2 = 0$ in \eqref{derivative-of-tau} survives and we get
 \begin{equation}
  \left.\frac{d^{d-1}}{dz^{d-1}} \tau_{p,p'}(y_1,z)\right|_{z_i} = f_d^{(d-1)}(1)\left[\frac{d}{dz}q_{p,p'}(y_1,z)|_{z_i}\right]^{d-1} = (-1)^{d-1}(d-1)!\left[\frac{d}{dz}q_{p,p'}(y_1,z)|_{z_i}\right]^{d-1}.
 \end{equation}
 Utilizing the fact that the derivative of a quadratic $q(x) = ax^2 + bx + c$ at its roots is given by $\pm\sqrt{b^2 - 4ac}$ combined with the formula for $q_{p,p'}$ \eqref{definition-of-q}, we get
 \begin{equation}
  \frac{d}{dz}q_{p,p'}(y_1,z)|_{z_i} = \pm 2\sqrt{(b_1'\cdot b_1')(b_1'\cdot b_2')^2-(b_2'\cdot b_2')(b_1'\cdot b_1')} = \pm 2\det(D_{p,p'})^{-\frac{1}{2}}.
 \end{equation}
 Taken together, this shows that the jump in the $d-1$-st derivative of $\tau_{p,p'}(y_1,z)$ at the breakpoints $z_1$ and $z_2$ has magnitude
 \begin{equation}\label{derivative-bound}
 \left|\left.\frac{d^{d-1}}{dz^{d-1}} \tau_{p,p'}(y_1,z)\right|_{z_i}\right| \lesssim_d \det(D_{p,p'})^{-\frac{d-1}{2}}.
 \end{equation}

 Going back to equation \eqref{inner-product-equation}, we see that due to the compact support of $\psi$, the integral in \eqref{inner-product-equation} is supported on a square with side length $2\delta$ in $y_1$ and $y_2$. To clarify this, we make the change of variables $s = \delta^{-1}y_1+2l$, $t = \delta^{-1}y_2+2l'$, and use that $\psi$ is supported on $[-1,1]$, to get (for notational convenience we let $y(s,l) = \delta (s-2l)$)
 \begin{equation}
  \langle g_{p,l}, g_{p',l'}\rangle_{H} = \delta^{2k+2}\det(D_{p,p'})^{-\frac{1}{2}} \int_{-1}^1 \int_{-1}^1 \psi(s)\psi(t) \tau_{p,p'}(y(s,l), y(t,l'))ds dt.
 \end{equation}
 We now estimate the sum over $l'$ as
  \begin{equation}\label{big-equation}
  \begin{split}
  \sum_{l'=-\frac{N}{2}}^{\frac{N}{2}} |\langle g_{p,l}, g_{p',l'}\rangle_{H}| &= \delta^{2k+2}\det(D_{p,p'})^{-\frac{1}{2}}\sum_{l'=-\frac{N}{2}}^{\frac{N}{2}}\left|\int_{-1}^1\int_{-1}^1 \psi(s)\psi(t)\tau_{p,p'}(y(s,l), y(t,l'))dsdt\right| \\
  &\leq \delta^{2k+2}\det(D_{p,p'})^{-\frac{1}{2}}\sum_{l'=-\frac{N}{2}}^{\frac{N}{2}}\int_{-1}^1\left|\int_{-1}^1 \psi(s)\psi(t)\tau_{p,p'}(y(s,l), y(t,l'))dt\right|ds \\
  & = \delta^{2k+2}\det(D_{p,p'})^{-\frac{1}{2}}\int_{-1}^1|\psi(s)|\sum_{l'=-\frac{N}{2}}^{\frac{N}{2}}\left|\int_{-1}^1 \psi(t)\tau_{p,p'}(y(s,l), y(t,l'))dt\right|ds.
  \end{split}
 \end{equation}
 For fixed $s$ and $l$, consider the inner sum
 \begin{equation}\label{sum-to-bound}
  \sum_{l'=-\frac{N}{2}}^{\frac{N}{2}}\left|\int_{-1}^1 \psi(t)\tau_{p,p'}(y(s,l), y(t,l'))dt\right| = \sum_{l'=-\frac{N}{2}}^{\frac{N}{2}}\left|\int_{-1}^1 \psi(t)\tau_{p,p'}(y(s,l), \delta (t - 2l'))dt\right|.
 \end{equation}
 In the integrals appearing in this sum, the variable $z = \delta (t - 2l')$ runs over the line segment $[\delta(2l'-1),\delta(2l'+1)]$. These segments are disjoint for distinct $l'$ and are each of length $2\delta$. 
 
 Further, recall that for fixed $y_1 = y(s,l)$, the function $\tau_{p,p'}(y_1, z)$ is a piecewise polynomial of degree $2d-2$ with at most two breakpoints $z_1$ and $z_2$. Combined with the fact that $2d-1$ moments of $\psi$ vanish, this implies that at most two terms in the above sum are non-zero, namely those where the corresponding integral contains a breakpoint.
 
 Furthermore, the bound on the jump in the $d-1$-st order derivatives at the breakpoints \eqref{derivative-bound} implies that in the intervals (of length $2\delta$) which contain a breakpoint, there exists a polynomial $q_i$ of degree $d-2$ for which
 \begin{equation}
  |\tau_{p,p'}(y_1,z) - q_i(z)| \leq \frac{(2\delta)^{d-1}}{(d-1)!}M_d \det(D_{p,p'})^{-\frac{d-1}{2}} \lesssim_d \delta^{d-1} \det(D_{p,p'})^{-\frac{d-1}{2}}
 \end{equation}
 on the given interval. Using again the vanishing moments of $\psi$, we see that the nonzero integrals in the sum \eqref{sum-to-bound} (of which there are at most $2$) satisfy
 $$
 \left|\int_{-1}^1 \psi(t)\tau_{p,p'}(y(s,l), \delta (t - 2l'))dt\right| \lesssim_{k,d} \delta^{d-1}\det(D_{p,p'})^{-\frac{d-1}{2}}.
 $$
 So for each fixed $s$ and $l$, we get the bound
 \begin{equation}
  \sum_{l'=-\frac{N}{2}}^{\frac{N}{2}}\left|\int_{-1}^1 \psi(t)\tau_{p,p'}(y(s,l), y(t,l'))dt\right| \lesssim_{k,d} \delta^{d-1} \det(D_{p,p'})^{-\frac{d-1}{2}}.
 \end{equation}
 Plugging this into equation \eqref{big-equation}, we get
 \begin{equation}
  \sum_{l'=-\frac{N}{2}}^{\frac{N}{2}} |\langle g_{p,l}, g_{p',l'}\rangle_{H}| \lesssim_{k,d} \delta^{2k+d+1}\det(D_{p,p'})^{-\frac{d}{2}}\int_{-1}^1|\psi(s)| ds \lesssim_{k,d} \delta^{2k+d+1}\det(D_{p,p'})^{-\frac{d}{2}}.
 \end{equation}
We analyze the $\det(D_{p,p'})^{-\frac{d}{2}}$ term using that $\omega_p$ and $\omega_{p'}$ are on the sphere to get
\begin{equation}
 \det(D_{p,p'})^{-\frac{d}{2}} = (1-\langle \omega_p,\omega_{p'}\rangle^2)^{-\frac{d}{2}} = \frac{1}{\sin(\theta_{p,p'})^d},
\end{equation}
where $\theta_{p,p'}$ represents the angle between $\omega_p$ and $\omega_{p'}$.

Summing over $p'\neq p$, we get
\begin{equation}\label{eq-1357}
 \sum_{(p',l')\neq (p,l)} |\langle g_{p,l}, g_{p',l'}\rangle_H| \lesssim_{k,d} \delta^{2k+d+1}\sum_{p'\neq p}\frac{1}{\sin(\theta_{p,p'})^d}.
\end{equation}
The final step is to bound the above sum. This is done in a relatively straightforward manner by noting that this sum is comparable to the following integral 
\begin{equation}
 \sum_{p'\neq p}\frac{1}{\sin(\theta_{p,p'})^d} \eqsim_d N^{d-1}\int_{P^{d-1}-B(p,r)} |x-p|^{-d}dx,
\end{equation}
where we are integrating over projective space minus a ball of radius $r \gtrsim_d N^{-1}$ around $p$. Integrating around this pole of order $d$ in the $d-1$ dimensional $P^{d-1}$, this gives
\begin{equation}
 \sum_{p'\neq p}\frac{1}{\sin(\theta_{p,p'})^d} \eqsim_d N^d.
\end{equation}
To be more precise, we present the detailed estimates in what follows.

We bound the sum over one hemisphere
\begin{equation}
 \sum_{0 < \theta_{p,p'}\leq \frac{\pi}{2}}\frac{1}{\sin(\theta_{p,p'})^d},
\end{equation}
and note that the sum over the other hemisphere can be handled in an analogous manner. To this end, we decompose this sum as
\begin{equation}\label{eq-1365}
 \sum_{0 < \theta_{p,p'}\leq \frac{\pi}{2}}\frac{1}{\sin(\theta_{p,p'})^d} = \sum_{0 < \theta_{p,p'}\leq \frac{\pi}{4}}\frac{1}{\sin(\theta_{p,p'})^d} + \sum_{\frac{\pi}{4} < \theta_{p,p'}\leq \frac{\pi}{2}}\frac{1}{\sin(\theta_{p,p'})^d}.
\end{equation}
For the second sum, we note that $\sin(\theta_{p,p'}) \geq \frac{1}{\sqrt{2}}$, and the number of terms is at most $n = N^{d-1}$, so that the second sum is $\lesssim N^{d-1}$. 

To bound the first sum in \eqref{eq-1365}, we rotate the sphere so that $\omega_p = (0,...,0,1)$ is the north pole. We then take the $\omega_{p'}$ for which $\theta_{p,p'}\leq \frac{\pi}{4}$ and project them onto the tangent plane at $\omega_p$. Specifically, this corresponds to the map $\omega_{p'} = (x_1,...x_{d-1},x_d)\rightarrow x_{p'} = (x_1,...x_{d-1})$, which removes the last coordinate. 

It is now elementary to check that this maps distorts distances by at most a constant (since the $\omega_{p'}$ are all contained in a spherical cap of radius $\frac{\pi}{4}$), i.e. that for $p'_1\neq p'_2$, we have
\begin{equation}
 |x_{p'_1} - x_{p'_2}| \leq |\omega_{p'_1} - \omega_{p'_2}| \lesssim |x_{p'_1} - x_{p'_2}|,
\end{equation}
and also that $\sin(\theta_{p,p'}) = |x_{p'}|$.

This allows us to write the first sum in \eqref{eq-1365} as
\begin{equation}
 \sum_{0 < \theta_{p,p'}\leq \frac{\pi}{4}}\frac{1}{\sin(\theta_{p,p'})^d} = \sum_{0<|x_{p'}|\leq \frac{1}{\sqrt{2}}}\frac{1}{|x_{p'}|^d},
\end{equation}
where by construction we have $|\omega_{p'_1} - \omega_{p'_2}| \gtrsim_d N^{-1}$ for $p'_1\neq p'_2$, and thus $|x_{p'_1} - x_{p'_2}|\gtrsim_d N^{-1}$ as well. In addition, $|\omega_p - \omega_{p'}| \gtrsim_d N^{-1}$ and thus also $|x_{p'}| \gtrsim_d N^{-1}$.

Now let $r\gtrsim_d N^{-1}$ be such that the balls of radius $r$ around each of the $x_{p'}$, and around $0$, are disjoint. Notice that since $|x|^{-d}$ is a subharmonic function on $\mathbb{R}^{d-1} / \{0\}$, we have
\begin{equation}
 \frac{1}{|x_{p'}|^d} \leq \frac{1}{|B(x_{p'},r)|}\int_{B(x_{p'},r)}|y|^{-d}dy \lesssim_d N^{d-1}\int_{B(x_{p'},r)}|y|^{-d}dy.
\end{equation}
Since all of the balls $B(x_{p'},r)$ are disjoint and are disjoint from $B(0,r)$, we get (note that these integrals are in $\mathbb{R}^{d-1}$)
\begin{equation}
 \sum_{0<|x_{p'}|\leq \frac{1}{\sqrt{2}}}\frac{1}{|x_{p'}|^d} \lesssim_d N^{d-1}\int_{r \leq |y| \leq \frac{\pi}{2} + r} |y|^{-d}dy \leq N^{d-1}\int_{r \leq |y|} |y|^{-d}dy \lesssim_d N^{d-1}r^{-1} \lesssim_d N^d.
\end{equation}
Plugging this into equation \eqref{eq-1365} and bounding the sum over the other hemisphere in a similar manner, we get
\begin{equation}
 \sum_{p'\neq p}\frac{1}{\sin(\theta_{p,p'})^d} \lesssim_d N^d.
\end{equation}
Using equation \eqref{eq-1357}, we finally obtain
\begin{equation}
 \sum_{(p',l')\neq (p,l)} |\langle g_{p,l}, g_{p',l'}\rangle_H| \lesssim_{k,d} \delta^{2k+d+1}N^d.
\end{equation}
Combined with the lower bound \eqref{lower-bound}, which gives $\|g_{p,l}\|_H^2 \gtrsim_{k,d} \delta^{2k+1}$ for all $(p,l)$, we see that by choosing the factor $a$ in $\delta = aN^{-1}$ small enough (independently of $N$, of course), we can guarantee that the conditions of Lemma \ref{kolmogorov-lower-bound-lemma} (and thus also Corollary \ref{entropy-lower-bound-corollary}) are satisfied.

Applying Corollary \ref{entropy-lower-bound-corollary}, we see that
\begin{equation}
 \epsilon_{n}(A) \geq \frac{\min_{(p,l)} \|g_{p,l}\|_H}{\sqrt{8n}} \gtrsim_{k,d} n^{-\frac{1}{2}} \delta^{\frac{2k+1}{2}} \gtrsim_{k,d,a} n^{-\frac{1}{2}}N^{-\frac{2k+1}{2}},
\end{equation}
where $n = N^d$ is the total number of functions $g_{p,l}$. We obtain a completely analogous result for the Bernstein widths as well. This finally gives (since $a$ is fixed depending only upon $k$ and $d$)
\begin{equation}
 \epsilon_{n}(A)\gtrsim_{k,d} n^{-\frac{1}{2}-\frac{2k+1}{2d}},~b_{n}(A)\gtrsim_{k,d} n^{-\frac{1}{2}-\frac{2k+1}{2d}}.
\end{equation}

Applying Lemma \ref{kolmogorov-lower-bound-lemma}, we get
\begin{equation}
 d_n(A) \geq \frac{1}{2}\min_{(p,l)} \|g_{p,l}\|_H \gtrsim_{k,d} \delta^{\frac{2k+1}{2}} \gtrsim_{k,d,a} N^{-\frac{2k+1}{2}}.
\end{equation}
Since $n = N^d$ is the total number of functions $g_{p,l}$ we get as before
\begin{equation}
 d_n(A) \gtrsim_{k,d}n^{-\frac{2k+1}{2d}}.
\end{equation}
The monotonicity of the entropy and $n$-widths extends this bound to all $n$. This completes the proof.
\end{proof}
We remark that in the case of ReLU$^k$ activation functions on the sphere, the high degree of symmetry allows the Kolmogorov $n$-widths to be determined exactly in terms of the spectrum of a kernel operator \cite{long2021linear,bach2017breaking}, which we briefly describe in an abstract form here.

Specifically, the abstract situation here consists of the convex hull of a dictionary $\mathbb{D} = \{g\cdot f_e:g\in G\}\subset H$, where $H$ is a Hilbert space, $G$ is a compact Hausdorff topological group of isometries on the space $H$, and $f_e\in H$ is a fixed element. 
One simple example of this framework is the case where $H = \mathbb{R}^d$, $f_e = e_1$ is the first unit basis vector, and $G = \mathbb{Z}_n$ is the cyclic group on $d$ elements. The action of $G$ on $\mathbb{R}^n$ is given by cyclically shifting the indices. For this example $B_1(\mathbb{D})$ is the unit ball of the $\ell^1$-norm in $\mathbb{R}^n$ and this approach can be used to calculate its Kolmogorov $n$-widths with respect to $\ell^2$ (see \cite{lorentz1996constructive}, chapter 14, for instance).

Another example, which we are primarily interested in here is where $H = L^2(S^{d-1})$, $f_e(x) = \sigma(x_1)\in L^2(S^{d-1})$ for an activation function $\sigma$, and the group $G=O_d$ is the group of orthogonal transformations on $\mathbb{R}^d$. The action of $g\in G$, $(g\cdot f)(x) = f(g^{-1}x)$ is given by rotating the function $f$. In this case the dictionary $\mathbb{D}$ is given by $\{\sigma(\omega\cdot x):~\omega\in S^{d-1}\}\subset L^2(S^{d-1})$. This is the situation which has been studied in \cite{bach2017breaking,long2021linear}.

In this situation we can lower bound the Kolmogorov $n$-widths by averaging over the group $G$. Let $x_1,...,x_n$ be an orthonormal basis of a subspace $X_n$ and let $d\mu$ denote the normalized Haar measure on $G$. Consider the average distance to $X_n$
\begin{equation}
 \mathbb{E}_{g\sim d\mu} d(f_g, X_n)^2 = \mathbb{E}_{g\sim d\mu} \left(\|f_g\|_H^2 - \sum_{i=1}^n \langle f_g, x_i\rangle_H^2\right),
\end{equation}
where to simplify notation we have written $f_g$ for $g\cdot f_e$.
Using the assumption that $G$ consists of isometries, we get
\begin{equation}\label{eq-1341}
 \mathbb{E}_{g\sim d\mu} d(f_g, X_n)^2 = \|f_e\|_H^2 - \sum_{i=1}^n \mathbb{E}_{g\sim d\mu} \langle f_g, x_i\rangle_H^2 = \|f_e\|_H^2 - \sum_{i=1}^n \langle x_i, T_{G} (x_i)\rangle_H,
\end{equation}
where the operator $T_{G}:H\rightarrow H$ is given by the average of rank $1$ operators:
\begin{equation}
 T_{G}(x) = \mathbb{E}_{g\sim d\mu} \langle f_g, x\rangle_Hf_g.
\end{equation}
From this formula, it is clear that $T_G$ is a self-adjoint, compact, $G$-invariant operator on $H$ with trace $\text{Tr}(T_G) = \|f_e\|_H^2$. If we let $\lambda_1 \geq \lambda_2 \geq \cdots$ denote the eigenvalues of the operator $T_G$, we get from \eqref{eq-1341} and the minimax characterization of the eigenvalues that for any $n$-dimensional subspace $X_n\subset H$
\begin{equation}
 \mathbb{E}_{g\sim d\mu} d(f_g, X_n)^2 \geq \|f_e\|_H^2 - \sum_{i=1}^n \lambda_i = \sum_{i=n+1}^\infty \lambda_i,
\end{equation}
with equality if $X_n$ is the space spanned by $\phi_1,...,\phi_n$, the eigenfunctions corresponding to the $n$ largest eigenvectors. Since a maximum bounds an average, this gives the following lower bound on the Kolmogorov $n$-widths
\begin{equation}
 d_n(B_1(\mathbb{D})) \geq \sqrt{\sum_{i=n+1}^\infty \lambda_i}.
\end{equation}
Furthermore, suppose that $\lambda_n > \lambda_{n+1}$ and $X_n$ is taken to be the space spanned by $\phi_1,...,\phi_n$. Since $\lambda_n > \lambda_{n+1}$ and $T_G$ is $G$-invariant, we have that $X_n$ must be a $G$-invariant subspace. This means that $d(f_g,X_n)$ doesn't depend upon $g$ and so the average and maximum coincide. Thus if $\lambda_n > \lambda_{n+1}$, we actually have equality above and so
\begin{equation}
 d_n(B_1(\mathbb{D})) = \sqrt{\sum_{i=n+1}^\infty \lambda_i}.
\end{equation}

In the case of shallow neural networks on the sphere considered in \cite{bach2017breaking,long2021linear}, the operator $T_G$ is given by integration against an appropriate kernel and the eigenvalues $\lambda_i$ can be explicitly calculated in the case where $\sigma=\text{ReLU}^k$ (this is done in \cite{bach2017breaking} and the result used to bound the Kolmogorov widths in \cite{long2021linear}). This method allows an accurate determination of the constants in the $n$-width rates as well.

Finally, we will prove the following general result, from which a bound on the Kolmogorov $n$-widths leads to a lower bound on the metric entropy (see also \cite{siegel2021improved} for a version of this argument, which we call the skewed simplex argument since we find a skewed image of the $\ell^1$-unit ball in our space).
\begin{proposition}
 Let $H$ be a Hilbert space and $A\subset H$ a symmetric, convex set. Then
 \begin{equation}
  \epsilon_n(A)_H \geq Cd_n(A)_Hn^{-\frac{1}{2}},
 \end{equation}
 for an absolute constant $C$.
\end{proposition}
\begin{proof}
 Let $\delta > 0$ and define a collection of elements $g_1,...,g_n\in A$ recursively as follows:
 \begin{equation}
  \|g_i - P_{i-1}g_i\|_H \geq (1-\delta)\sup_{g\in A}\|g - P_{i-1}g\|_H,
 \end{equation}
 where $P_{i-1}$ is the orthogonal projection onto the span of $g_1,...,g_{i-1}$. By definition of the $n$-widths, we have
 \begin{equation}
  \|g_i - P_{i-1}g_i\|_H \geq (1-\delta)d_i(A)_H \geq (1-\delta)d_n(A)_H.
 \end{equation}
 Let $\tilde{g}_1,...,\tilde{g_n}$ be the Gram-Schmidt orthogonalization of $g_1,...,g_n$. Since the change of basis between $g_1,...,g_n$ and $\tilde{g}_1,...,\tilde{g_n}$ is upper triangular with ones on the diagonal, the volume (viewed in the $n$-dimensional Euclidean space spanned by $g_1,...,g_n$) of the convex hull of $g_1,...,g_n$ and $\tilde{g}_1,...,\tilde{g_n}$ is the same. Since $\tilde{g}_1,...,\tilde{g_n}$ are orthogonal with length at least $(1-\delta)d_n(A)_H$, we get (from the volume of the $\ell^1$-unit ball)
 \begin{equation}
  |\text{co}(g_1,...,g_n)| \geq ((1-\delta)d_n(A)_H)^n\frac{2^n}{n!}.
 \end{equation}
 Using the covering definition of the entropy and comparing this with the volume of $2^n$ balls of radius $\epsilon := \epsilon_n(A)_H$, we get
 \begin{equation}
  ((1-\delta)d_n(A)_H)^n\frac{2^n}{n!} \leq |\text{co}(g_1,...,g_n)| \leq (2\epsilon)^n\frac{\pi^{n/2}}{\Gamma\left(\frac{n}{2}+1\right)}.
 \end{equation}
 Utilizing Sterling's formula, we get
 \begin{equation}
  \epsilon \geq C(1-\delta)d_n(A)_Hn^{-\frac{1}{2}},
 \end{equation}
 for an absolute constant $C$.
 Letting $\delta\rightarrow 0$ completes the proof.
\end{proof}
Although the preceding method is simpler and allows a more precise estimate of the constants in the $n$-width and entropy rates for $B_1(\mathbb{D})$, we note that Theorem \ref{lower-bound-theorem} is more general. Specifically, it does not require the high degree of symmetry that the preceding argument does and thus applies to more general domains $\Omega$ and dictionaries $\mathbb{D}$. In addition, Theorem \ref{lower-bound-theorem} finds a collection of nearly orthogonal vectors as opposed to a (potentially highly) skewed image of the simplex within the set $B_1(\mathbb{D})$. This stronger condition enables us to obtain a lower bound on the Bernstein widths as well.

\subsection{Lower bounds on approximation rates for shallow neural networks}\label{consequences-section}
In this section, we use Theorem \ref{lower-bound-theorem} to obtain lower bounds on the approximation rates of shallow neural networks. The key is the following relationship between metric entropy and non-linear approximation rates, which can be viewed as an analogue of Carl's inequality \cite{carl1981entropy}. 
\begin{theorem}\label{dictionary-carls-theorem}
 Let $X$ be a Banach space and $\mathbb{D}\subset X$ a dictionary with $K_\mathbb{D}:=\sup_{h\in \mathbb{D}} \|h\|_X < \infty$. Suppose that for some constants $0 < l < \infty$, $C < \infty$, the dictionary $\mathbb{D}$ can be covered by $C\epsilon^{-l}$ sets of diameter $\epsilon$ for any $\epsilon > 0$. If there exists an $M,K < \infty$ and $\alpha > 0$ such that for all $f\in B_1(\mathbb{D})$
 \begin{equation}\label{approx-bound-estimate}
  \inf_{f_n\in \Sigma^\infty_{n,M}(\mathbb{D})} \|f - f_n\|_X \leq Kn^{-\alpha},
 \end{equation}
 then the entropy numbers of $B_1(\mathbb{D})$ are bounded by
 \begin{equation}
  \epsilon_{n\log{n}}(B_1(\mathbb{D}))_X \lesssim n^{-\alpha},
 \end{equation}
 where the implied constant is independent of $n$.
\end{theorem}
Thus, a given approximation rate from the set $\Sigma_{n,M}^\infty(\mathbb{D})$ implies a corresponding bound on the metric entropy.

Note that we are considering approximation by the set $\Sigma^\infty_{n,M}(\mathbb{D})$, defined in \eqref{dictionary-expansion-definition}, which corresponds to expansions with coefficients bounded in $\ell^\infty$. This is in contrast to previous results \cite{klusowski2018approximation,makovoz1996random} which obtained lower bounds when the coefficients were bounded in $\ell^1$. 

For the dictionaries $\mathbb{P}_k^d$, i.e. for ReLU$^k$ networks, the set $\Sigma^\infty_{n,M}(\mathbb{P}_k^d)$ corresponds to shallow neural networks with $n$ neurons, inner coefficients bounded, and outer coefficients bounded in $\ell^\infty$. For the dictionary $\mathbb{D}_\sigma := \{\sigma(\omega\cdot x + b)~\omega\in \mathbb{R}^d,~b\in \mathbb{R}\}$ where $\sigma$ is a sigmoidal activation function, the set $\Sigma^\infty_{n,M}(\mathbb{D}_\sigma)$ corresponds to shallow neural networks with $n$ neurons and outer coefficients bounded in $\ell^\infty$, with no bound on the inner coefficients. 
\begin{proof}
 In what follows, all implied constants will be independent of $n$.
 
 Let $n \geq 1$ be an integer. We use our assumption on $\mathbb{D}$ and set $\epsilon = n^{-\alpha-1}$. Then that there is a subset $\mathcal{D}_n\subset \mathbb{D}$ such that $|\mathcal{D}_n| \leq Cn^{(\alpha+1) l}$ and
 \begin{equation}
  \sup_{d\in \mathbb{D}} \inf_{s\in \mathcal{D}_n} \|d - s\|_H \leq \epsilon = n^{-\alpha-1}.
 \end{equation}
 
 The next step, in which the argument differs from that in \cite{makovoz1996random,klusowski2018approximation}, is to cover the unit ball in $\ell_\infty^n$ by unit balls in $\ell_1^n$ of radius $\delta$. Indeed, denoting a ball of radius $R$ in a Banach space $Y$ by $B_R(Y) = \{x\in Y:~|x|_Y\leq R\}$, we see that 
 \begin{equation}
 B_1(\ell^n_\infty) \subset B_n(\ell^n_1).
 \end{equation}
 
 Furthermore, we can cover the unit ball in a space $Y$ by $(1+\frac{2}{\delta})^n$ balls in $Y$ of radius $\delta$ (see \cite{pisier1999volume}, page $63$). Applying this to $Y=\ell^n_1$ and scaling the unit ball appropriately, we see that we can cover 
 \begin{equation}
  B_M(\ell^n_\infty) \subset B_{Mn}(\ell^n_1)
 \end{equation}
 by $(1+\frac{2Mn}{\delta})^n$ $\ell^n_1$-balls of radius $\delta$. Now we set $\delta = 2Mn^{-\alpha}$, so the number of balls will be at most 
 $$(1+n^{\alpha+1})^n = n^{(\alpha+1)n}(1 + n^{-(\alpha+1)})^n \lesssim n^{(\alpha+1)n},$$
 where the last inequality is due to $\alpha > 0$. Denote by $\mathcal{L}_n$ the centers of these balls.

 Denote by $\mathcal{S}_n$ the set of all linear combinations of $n$ elements of $\mathcal{D}_{n}$ with coefficients in $\mathcal{L}_{n}$. Then clearly 
 \begin{equation}\label{eq-707}
 |\mathcal{S}_n| \leq |\mathcal{D}_{n}|^n|\mathcal{L}_{n}| \lesssim C^nn^{(\alpha+1) ln}n^{(\alpha+1)n} = C^nn^{(\alpha+1)(l+1)n}.
 \end{equation}
 
 By \eqref{approx-bound-estimate}, we have for every $f\in B_1(\mathbb{D})$ an $f_n\in \Sigma_{n,M}(\mathbb{D})$ such that
 \begin{equation}
  f_n = \sum_{j=1}^n a_jh_j
 \end{equation}
 and $\|f - f_n\|_X \lesssim n^{-\alpha}$, $h_j\in \mathbb{D}$ and $|a_j| \leq M$ for each $j$. 
 
 We now replace the $h_j$ by their closest elements in $\mathcal{D}_{n}$ and the coefficients $a_j$ by their closest point in $\mathcal{L}_{n}$. Since $\|h_j\|_H\leq K_\mathbb{D}$ and $|a_j| \leq M$ for each $j$, this results in a point $\tilde{f}_n\in \mathcal{S}_n$ with 
 $$\|f_n - \tilde{f}_n\|_H \leq Mn\epsilon + K_\mathbb{D}\delta = Mn^{-\alpha} + 2K_\mathbb{D}M n^{-\alpha} \lesssim n^{-\alpha}.$$ Thus $\|f - \tilde{f}_n\|_H\lesssim n^{-\alpha}$ and so
 \begin{equation}
  \epsilon_{\log{|\mathcal{S}_n|}} \lesssim n^{-\alpha}.
 \end{equation}
 By equation \eqref{eq-707}, we see that $\log{|\mathcal{S}_n|} \lesssim n\log{n}$, which completes the proof.
\end{proof}

Using Theorem \ref{dictionary-carls-theorem} and Theorem \ref{lower-bound-theorem} we can immediately conclude the following lower bound on the approximation rates by neural networks with ReLU$^k$ activation function.
\begin{corollary}
 Let $k\geq 0$ and $M < \infty$ be fixed and suppose that $\alpha > \frac{1}{2}+\frac{2k+1}{2d}$. Then
 \begin{equation}\label{approximation-rate-condition}
  \sup_{n \geq 1} n^{\alpha}\left[\sup_{f\in B_1(\mathbb{P}_k^d)}\inf_{f_n\in \Sigma_{n,M}^\infty(\mathbb{P}_k^d)} \|f - f_n\|_{L^2(\Omega)}\right] = \infty.
 \end{equation}
\end{corollary}
This corollary shows that the exponent in the approximation rates for shallow ReLU$^k$ neural networks with respect to the variation norm cannot be improved beyond $-\frac{1}{2} - \frac{2k+1}{2d}$, even if the $\ell^1$ bound on the outer coefficients is relaxed to an $\ell^\infty$ bound.
\begin{proof}
 From the theory developed in Section \ref{main-result-1-section}, it is clear that the dictionaries $\mathbb{P}_k^d$ satisfy the assumptions of Theorem \ref{dictionary-carls-theorem} since they are smoothly parameterized by compact manifolds. If the supremum in \eqref{approximation-rate-condition} were finite, then by Theorem \ref{dictionary-carls-theorem} we would have $\epsilon_{n\log n}(B_1(\mathbb{P}_k^d)) \lesssim n^{-\alpha}$. This contradicts the lower bound from Theorem \ref{lower-bound-theorem} since $\alpha > \frac{1}{2}+\frac{2k+1}{2d}$.
\end{proof}

Next, we extend this result to sigmoidal activation function with bounded variation. For this, we need the following technical lemma.
\begin{lemma}\label{sigmoidal-lemma}
 Let $\Omega \subset \mathbb{R}^d$ be a bounded domain and suppose that $\sigma$ is a sigmoidal function with bounded variation. Then there exist $C,l < \infty$ such that the dictionary
 \begin{equation}
  \mathbb{D}_\sigma := \left\{\sigma(\omega\cdot x + b),~\omega\in \mathbb{R}^d,~b\in \mathbb{R}\right\}
 \end{equation}
 can be covered by $C\epsilon^{-l}$ balls of radius $\epsilon$ in $L^2(\Omega)$. In particular, $\mathbb{D}_\sigma$
 satisfies the assumptions of Theorem \ref{dictionary-carls-theorem}.
\end{lemma}
This result generalizes Lemma 2 in \cite{makovoz1996random} by relaxing the assumption on $\sigma$. Instead of requiring a Lipschitz condition and the assumption that $\sigma$ approaches the Heaviside $\sigma_0$ at a polynomial rate, we only require the activation function $\sigma$ to have bounded variation.
\begin{proof}
 Consider the Jordan decomposition of the function $\sigma = \sigma^+ - \sigma^-$, where $\sigma^+$ and $\sigma^-$ are non-decreasing functions and
 \begin{equation}\label{bv-condition}
  \|\sigma\|_{BV} = \lim_{x\rightarrow\infty} (\sigma^+(x) + \sigma^-(x)) - \lim_{x\rightarrow-\infty} (\sigma^+(x) + \sigma^-(x)) < \infty.
 \end{equation}
 Denote by $a^+ := \lim_{x\rightarrow -\infty} \sigma^+(x)$ and $b^+ := \lim_{x\rightarrow \infty} \sigma^+(x)$ and likewise for $\sigma^-$. By \eqref{bv-condition} $a^+,b^+,a^-$ and $b^-$ are all finite. Further, $[a^+,b^+]$ is the closure of the range of $\sigma^+$ and $[a^-,b^-]$ is the closure of the range of $\sigma^-$.
 
 We proceed to divide the intervals $[a^+,b^+]$ and $[a^-,b^-]$ into intervals of length at most $\frac{\epsilon}{2}$. Denote these intervals by $[x_{i-1},x_{i})$ and $[y_{i-1},y_{i})$ where we have
 \begin{equation}
  a^+ = x_0 < \cdots < x_{n_1} = b^+
 \end{equation}
 and
 \begin{equation}
  a^- = y_0 < \cdots < y_{n_2} = b^-.
 \end{equation}
 
 This partitions the domain $\mathbb{R}$ into two sets of disjoint intervals: $\sigma^{-1}([x_{i-1},x_i))$ for $i=1,...,n_1$ and $\sigma^{-1}([y_{i-1},y_i))$ for $i=1,...,n_2$. Take the common refinement of these intervals, i.e. consider nonempty all intervals of the form $\sigma^{-1}([x_{i-1},x_i)) \cap \sigma^{-1}([y_{j-1},y_j))$ and define a piecewise constant function $\sigma_\epsilon$ by
 \begin{equation}
  \sigma_\epsilon(x) = x_{i-1} + y_{j-1}~\text{if $x\in \sigma^{-1}([x_{i-1},x_i)) \cap \sigma^{-1}([y_{j-1},y_j))$}.
 \end{equation}
 By construction, we have for any $x$ that
 \begin{equation}
  |\sigma_\epsilon(x) - \sigma(x)| \leq |x_{i-1} - \sigma^+(x)| + |y_{j-1} - \sigma^-(x)| \leq \frac{\epsilon}{2} + \frac{\epsilon}{2} = \epsilon,
 \end{equation}
 since $\sigma^+(x)\in [x_{i-1},x_i)$ and $\sigma^-(x)\in [y_{j-1},y_j)$. Thus, we have
 \begin{equation}\label{eq-674}
  \|\sigma_\epsilon(\omega\cdot x + b) - \sigma(\omega \cdot x + b)\|_{L^2(\Omega,dx)} \leq |\Omega|^{\frac{1}{2}}\epsilon
 \end{equation}
 uniformly in $\omega,b\in \mathbb{R}^d\times \mathbb{R}$. 
 
 In addition, it is easy to see that there are points $z_0 < z_1 < \cdots < z_n$ with $n = n_1 + n_2 \lesssim \epsilon^{-1}$ such that $\sigma_\epsilon$ is constant on $(z_i,z_{i+1})$ and on $(-\infty, z_0)$ and $(z_n,\infty)$.
 
 Next, choose an $\epsilon^3$-net for the slightly enlarged domain
 \begin{equation}
     \Omega_\epsilon = \{x:\text{dist}(x,\Omega) \leq \epsilon^3\},
 \end{equation}
 which will contain at most $N \lesssim \epsilon^{-3d}$ points $x_1,...,x_N\in \Omega_\epsilon$. For each $x_i$ and each $z_j$ consider the hyperplane in the parameter space $\mathbb{R}^d\times \mathbb{R}$ given by
 \begin{equation}
  H_{ij} = \{(\omega,b)\in \mathbb{R}^d\times \mathbb{R}:~\omega\cdot x_i + b = z_j\}.
 \end{equation}
 It is well-known that $K$ hyperplanes in $\mathbb{R}^{d+1}$ cut the space $\mathbb{R}^{d+1}$ into at most
 \begin{equation}
  \sum_{i=0}^{d+1}\binom{K}{i} \leq K^{d+1}
 \end{equation}
 regions. Thus the hyperplanes $H_{ij}$ cut the parameter space $\mathbb{R}^d\times \mathbb{R}$ into at most
 \begin{equation}
  M = (nN)^{d+1} \lesssim \epsilon^{-(3d+1)(d+2)}
 \end{equation}
 regions $R_1,...,R_M$.
 
 We claim that for each $i=1,...,M$, the set
 \begin{equation}
  S_i := \left\{\sigma(\omega\cdot x + b):~(\omega,b)\in R_i\right\}
 \end{equation}
 is contained in a ball of radius $r\lesssim \epsilon$ in $L^2(\Omega,dx)$. Setting $l = (3d+1)(d+2)$ and choosing $C$ appropriately large, we obtain the desired result.
 
 Fix $(\omega,b),(\omega',b')\in R_i$. From the triangle inequality and equation \eqref{eq-674} we see that
 \begin{equation}\label{triangle-inequality-702}
 \begin{split}
  \|\sigma(\omega\cdot x + b) - \sigma(\omega' \cdot x + b')\|_{L^2(\Omega,dx)} &\leq \|\sigma(\omega\cdot x + b) - \sigma_\epsilon(\omega \cdot x + b)\|_{L^2(\Omega,dx)} \\
  &+ \|\sigma_\epsilon(\omega\cdot x + b) - \sigma_\epsilon(\omega' \cdot x + b')\|_{L^2(\Omega,dx)} \\
  &+ \|\sigma_\epsilon(\omega'\cdot x + b') - \sigma(\omega' \cdot x + b')\|_{L^2(\Omega,dx)} \\
  & \leq 2|\Omega|^{\frac{1}{2}}\epsilon + \|\sigma_\epsilon(\omega\cdot x + b) - \sigma_\epsilon(\omega' \cdot x + b')\|_{L^2(\Omega,dx)}.
  \end{split}
 \end{equation}
 To conclude the proof, we bound the difference
 \begin{equation}
  \|\sigma_\epsilon(\omega\cdot x + b) - \sigma_\epsilon(\omega' \cdot x + b')\|^2_{L^2(\Omega,dx)} = \int_\Omega (\sigma_\epsilon(\omega\cdot x + b) - \sigma_\epsilon(\omega' \cdot x + b'))^2 dx.
 \end{equation}
 For this, we consider the set
 \begin{equation}
  D = \{x\in \Omega:~  \sigma_\epsilon(\omega\cdot x + b) \neq \sigma_\epsilon(\omega' \cdot x + b')\}.
 \end{equation}
 From the definition of $\sigma_\epsilon$, we see that $x\in D$ only if there exists a $z_j$ such that
 \begin{equation}
  \omega\cdot x + b \leq z_j \leq \omega' \cdot x + b',
 \end{equation}
 or vice versa (i.e. with the order reversed). Thus we have
 \begin{equation}\label{D-union-equation}
  D \subset \bigcup_{j=0}^n D^{+}_j\cup \bigcup_{j=0}^n D^{-}_j,
 \end{equation}
 where
 \begin{equation}
  D^{+}_{j} = \{\omega\cdot x + b \leq z_j\}\cap \{z_j \leq \omega' \cdot x + b'\}
 \end{equation}
 and
 \begin{equation}
  D^{-}_{j} = \{\omega\cdot x + b \geq z_j\}\cap \{z_j \geq \omega' \cdot x + b'\}.
 \end{equation}
 By construction none of the sets $D^{\pm}_{j}$ contain any of the points $x_1,...,x_N$ since $(\omega,b)$ and $(\omega',b')$ are both in the same region $R_i$. Since $x_1,...,x_N$ forms an $\epsilon^3$-net for $\Omega_\epsilon$, this implies that none of the $D^{\pm}_{j}\cap \Omega_\epsilon$ can contain a ball of radius $\epsilon^3$. Consider the sets 
 \begin{equation}
     \Sigma_j := \{x\in \Omega:~\text{dist}(x, \{y:~\omega\cdot y + b = z_j\}) \leq \epsilon^3\}~\text{and}~\Sigma_j' := \{x\in \Omega:~\text{dist}(x, \{y:~\omega'\cdot y + b' = z_j\}) \leq \epsilon^3\},
 \end{equation}
 which are strips of width $\epsilon^3$ around the hyperplanes defined by $\omega\cdot y + b = z_j$ and $\omega'\cdot y + b' = z_j$ intersected with $\Omega$, respectively.
 We claim that
 \begin{equation}\label{D-containment-equation}
     D^{\pm}_{j}\cap \Omega \subset \Sigma_j \cup \Sigma_j',
 \end{equation}
 for each $j$ and choice of sign $\pm$.
 Suppose to the contrary that for some $j$ there exists an $x\in D^{+}_{j}\cap \Omega$ (the case of negative sign is exactly the same) such that
 \begin{equation}
     \text{dist}(x, \{y:~\omega\cdot y + b = z_j\}) > \epsilon^3~\text{and}~\text{dist}(x, \{y:~\omega'\cdot y + b' = z_j\}) > \epsilon^3.
 \end{equation}
 These two conditions imply that the ball of radius $\epsilon^3$ about $x$ is contained in $D^{+}_{j}$. Further, since $x\in \Omega$, this ball is also contained in $\Omega_\epsilon$. But $D^{+}_{j} \cap \Omega_\epsilon$ cannot contain a ball of radius $\epsilon^3$. This contradiction shows that \eqref{D-containment-equation} holds. From this we deduce that
 \begin{equation}
     |D^{+}_{j}\cap \Omega| \leq |\Sigma_j| + |\Sigma_j'| \lesssim \epsilon^3,
 \end{equation}
 since $\Sigma_j$ and $\Sigma'_j$ are strips of width $\epsilon^3$ and $\Omega$ is a bounded domain. Using \eqref{D-union-equation} and a union bound, we obtain
 \begin{equation}
     |D| \lesssim n\epsilon^3\lesssim \epsilon^2.
 \end{equation}
 Finally, the difference $\sigma_\epsilon(\omega\cdot x + b) - \sigma_\epsilon(\omega' \cdot x + b')$ is equal to $0$ outside of $D$ and on $D$ it is bounded by $\sup_x \sigma(x) - \inf_x \sigma(x) \leq \|\sigma\|_{BV} \lesssim 1$. This implies that
 \begin{equation}
  \int_\Omega (\sigma_\epsilon(\omega\cdot x + b) - \sigma_\epsilon(\omega' \cdot x + b'))^2 dx \lesssim \epsilon^2,
 \end{equation}
 and finally that
 \begin{equation}
  \|\sigma_\epsilon(\omega\cdot x + b) - \sigma_\epsilon(\omega' \cdot x + b')\|_{L^2(\Omega,dx)} \lesssim \epsilon.
 \end{equation}
 Using \eqref{triangle-inequality-702} and that $(\omega,b),(\omega',b')\in S_i$ were arbitrary, we see that the diameter of the sets $S_i$ is $\lesssim \epsilon$, which completes the proof.
\end{proof}

Using Lemma \ref{sigmoidal-lemma}, we show that the lower bound on the approximation rates holds even for a sigmoidal activation function with bounded variation.
\begin{corollary}
 Let $\Omega\subset \mathbb{R}^d$ be a bounded domain and $\sigma$ be a sigmoidal activation function with bounded variation. Consider the dictionary $\mathbb{D}_\sigma^d\subset L^2(B_1^d)$ defined in \eqref{sigma-dictionary-definition}. Then for any $M < \infty$ and $\alpha > \frac{1}{2}+\frac{1}{2d}$ we have
\begin{equation}\label{approximation-rate-condition-sigmoidal}
  \sup_{n \geq 1} n^{\alpha}\left[\sup_{f\in B_1(\mathbb{D}_\sigma)}\inf_{f_n\in \Sigma_{n,M}^\infty(\mathbb{D}_\sigma)} \|f - f_n\|_{L^2(B_1^d)}\right] = \infty.
 \end{equation}
\end{corollary}
This shows that the exponent in the approximation rate derived by Makovoz \cite{makovoz1996random} is optimal, even if the outer coefficients of the network are only bounded in $\ell^\infty$ and the activation function is a general sigmoidal function with bounded variation.
\begin{proof}
 We observe that since $\sigma$ is a sigmoidal activation function and $\Omega$ is a bounded domain, we have 
 \begin{equation}
     \lim_{a\rightarrow \infty} \|\sigma(a(\omega\cdot x + b)) - \sigma_0(\omega\cdot x +b)\|_{L^2(\Omega)} = 0,
 \end{equation}
 where we recall that $\sigma_0$ is the Heaviside activation function. Since
 \begin{equation}
     \sigma(a(\omega\cdot x + b)) \in \mathbb{D}_\sigma
 \end{equation}
 for every $a \in \mathbb{R}$, this implies that $B_1(\mathbb{D}_\sigma) \supset B_1(\mathbb{P}_0^d)$.
 By Lemma \ref{sigmoidal-lemma} and Theorem \ref{dictionary-carls-theorem}, if the supremum in \eqref{approximation-rate-condition-sigmoidal} were finite, then the metric entropy would satisfy $$\epsilon_{n\log n}(B_1(\mathbb{P}_k^d)) \leq \epsilon_{n\log n}(B_1(\mathbb{D}_\sigma)) \lesssim n^{-\alpha}.$$
 This contradicts the lower bound from Theorem \ref{lower-bound-theorem} since $\alpha > \frac{1}{2}+\frac{1}{2d}$.
\end{proof}

\section{Conclusion}
We have introduced the notion of a smoothly parameterized dictionary and have bounded both approximation rates and fundamental quantities such as the metric entropy and $n$-widths for convex hulls of such dictionaries. Further, we have developed a method for lower bounding $n$-widths and metric entropy of convex hulls of certain classes of ridge functions. Applying these results to shallow neural networks, we obtain sharp approximation rates for neural networks with ReLU$^k$ activation functions, improving upon several results in the literature. In addition, this allows us to compare ReLU$^k$ networks with other methods and to show that they are optimal on their corresponding variation space. 

There are a few further questions we would like to propose. First, it is unclear how to compute entropy or $n$-width bounds on $B_1(\mathbb{D})$, and specifically $B_1(\mathbb{P}_k^d)$, in $L^p$ for $p\neq 2$. For this problem partial results appear in \cite{klusowski2018approximation,makovoz1998uniform,bach2017breaking}, but a complete solution seems to require significant new ideas. Second, we have been primarily interested in the rates for fixed dimension in this work and have not taken care to precisely determine the implied constants. As such, our work it mainly interesting for problems in fixed moderate dimension. Obtaining tighter bounds on the constants will be important in quantifying the curse of dimensionality. Finally, we would like to extend this theory to approximation by deeper neural networks. 
\section{Acknowledgements}
We would like to thank Professors Russel Caflisch, Ronald DeVore, Weinan E, Albert Cohen, Stephan Wojtowytsch, Jason Klusowski and Lei Wu for helpful discussions. This work was supported by the Verne M. Willaman Chair Fund at the Pennsylvania State University, and the National Science Foundation (Grant No. DMS-1819157 and DMS-2111387).

\bibliographystyle{amsplain}
\bibliography{refs}

\end{document}